\documentclass[11pt]{article}

\usepackage{epsfig,epsf,fancybox}
\usepackage{amsmath}
\usepackage{mathrsfs}
\usepackage{amssymb}
\usepackage{graphicx}
\usepackage{color}
\usepackage{multirow}
\usepackage{paralist}
\usepackage{verbatim}
\usepackage{galois}
\usepackage{algorithm}
\usepackage{algorithmic}
\usepackage{boxedminipage}
\usepackage{booktabs}
\usepackage{accents}
\usepackage{stmaryrd}
\usepackage{tikz}
\usetikzlibrary{calc}
\usepackage{subcaption}
\usepackage[bottom]{footmisc}
\usepackage{natbib}
\usepackage{tablefootnote}
\usepackage{bm}

\usepackage[colorlinks,linkcolor=magenta,citecolor=blue, pagebackref=true,backref=true]{hyperref}
\renewcommand*{\backrefalt}[4]{%
    \ifcase #1 \footnotesize{(Not cited.)}%
    \or        \footnotesize{(Cited on page~#2.)}%
    \else      \footnotesize{(Cited on pages~#2.)}%
    \fi}

\long\def\comment#1{}
\usepackage{nicefrac}

\newtheorem{theorem}{Theorem}[section]
\newtheorem{corollary}[theorem]{Corollary}
\newtheorem{lemma}[theorem]{Lemma}

\newtheorem{definition}{Definition}[section]
\newtheorem{example}{Example}[section]

\newtheorem{remark}{Remark}[section]
\newtheorem{assumption}{Assumption}[section]

\newcommand{\secref}[1]{Section~\ref{#1}}

\renewcommand{\eqref}[1]{Eq.~(\ref{#1})}

\newcommand{\thmref}[1]{Theorem~\ref{#1}}

\numberwithin{equation}{section}

\newcommand{\reals}{\mathbb{R}}

\newcommand{\E}{\mathbb{E}}
\newcommand{\NN}{\mathbb{N}}

\newcommand{\inner}[1]{\left\langle#1\right\rangle}

\newcommand{\EE}{\mathbb{E}}

\newcommand{\conv}{\textnormal{conv}}

\newcommand{\bb}{\mathbf b}
\newcommand{\x}{\mathbf x}
\newcommand{\y}{\mathbf y}
\newcommand{\s}{\mathbf s}

\newcommand{\g}{\mathbf g}
\newcommand{\h}{\mathbf h}
\newcommand{\e}{\mathbf e}
\newcommand{\z}{\mathbf z}
\newcommand{\w}{\mathbf w}
\newcommand{\W}{\mathbf W}
\newcommand{\su}{\mathbf u}
\newcommand{\sv}{\mathbf v}
\newcommand{\sxi}{\boldsymbol{\xi}}

\newcommand{\XCal}{\mathcal{X}}

\newcommand{\zero}{\mathbf{0}}

\newcommand{\br}{\mathbb{R}}

\newcommand{\Scal}{\mathcal{S}}

\newcommand{\norm}[1]{\|{#1} \|}

\newcommand{\calC}{\mathcal{C}}

\newcommand{\softrelu}{\mathrm{softrelu}}
\newcommand{\relu}{\mathrm{relu}}

\newcommand{\calR}{\mathcal{R}}

\newcommand{\Exp}{\mathbb{E}}

\renewcommand{\Pr}{\mathbb{P}}

\newcommand{\poly}{\mathrm{poly}}

\title{\bf{Deterministic Nonsmooth Nonconvex Optimization}}

\author{
\begin{tabular}{ccccc}
Michael I. Jordan$^{*}$
& Guy Kornowski$^\ddagger$
& Tianyi Lin$^{*}$
& Ohad Shamir$^{\ddagger}$
& Manolis Zampetakis$^{*}$
\end{tabular}
\vspace*{.2in}
\\
$^{*}$University of California, Berkeley \\
$^\ddagger$Weizmann Institute of Science
}
\date{}

\begin{document}

\maketitle

\begin{abstract}
We study the complexity of optimizing nonsmooth nonconvex Lipschitz functions by producing $(\delta,\epsilon)$-Goldstein stationary points. 
Several recent works have presented randomized algorithms that produce such points using $\widetilde{O}(\delta^{-1}\epsilon^{-3})$ first-order oracle calls, independent of the dimension $d$.  It has been an open problem as to whether a similar result can be obtained via a deterministic algorithm. We resolve this open problem, showing that randomization is necessary to obtain a dimension-free rate. In particular, we prove a lower bound of $\Omega(d)$ for any deterministic algorithm. 
Moreover, we show that unlike smooth or convex optimization, access to function values is required for any deterministic algorithm to halt within any finite time horizon.

On the other hand, we prove that if the function is even slightly smooth, then the dimension-free rate of $\widetilde{O}(\delta^{-1}\epsilon^{-3})$ can be obtained by a deterministic algorithm with merely a logarithmic dependence on the smoothness parameter.
Motivated by these findings, we turn to study the complexity of deterministically smoothing Lipschitz functions. Though there are well-known efficient black-box randomized smoothings, we start by showing that no such deterministic procedure can smooth functions in a meaningful manner (suitably defined), resolving an open question in the literature.
We then bypass this impossibility result for the structured case of ReLU neural networks. To that end, in a practical ``white-box'' setting in which the optimizer is granted access to the network's architecture, we propose a simple, dimension-free, deterministic smoothing of ReLU networks that provably preserves $(\delta,\epsilon)$-Goldstein stationary points. Our method applies to a variety of architectures of arbitrary depth, including ResNets and ConvNets.
Combined with our algorithm for slightly-smooth functions, this yields the first deterministic, dimension-free algorithm for optimizing ReLU networks, circumventing our lower bound.
\end{abstract}

\section{Introduction}

We consider the problem of optimizing a Lipschitz continuous function, $f:\reals^d\to\reals$, which is potentially not smooth nor convex, using a first-order algorithm which utilizes values and derivatives of the function at various points. The theoretical analysis of nonsmooth and nonconvex optimization has long been a focus of research in economics, control theory and computer science~\citep{Clarke-1990-Optimization, Makela-1992-Nonsmooth, Outrata-1998-Nonsmooth}. In recent years, this area has received renewed attention stemming from the fact that essentially all optimization problems associated with training modern neural networks are neither smooth nor convex, due to their depth and the ubiquitous use of rectified linear units (ReLUs), among other nonsmooth components \citep{Nair-2010-Rectified, Glorot-2011-Deep}.

Since the minimization of a Lipschitz function $f$ is well known to be intractable~\citep{Nemirovski-1983-Problem, Murty-1987-Some, Nesterov-2018-Lectures}, a local measure of optimality is required in order to obtain any reasonable guarantees. Accordingly, it is common to make use of the generalized gradient $\partial f(x)$ due to \citet{Clarke-1974-Necessary, Clarke-1975-Generalized, Clarke-1981-Generalized}, which is a natural generalization of the gradient and the convex subgradient \citep{Clarke-2008-Nonsmooth, Rockafellar-2009-Variational, Burke-2020-Gradient}, and seek points with small subgradient. Although under certain regularity assumptions it is possible to \emph{asymptotically} find an approximate Clarke stationary point of $f$,\footnote{Namely, $\x\in\br^d$ such that $\min\{\|\g\|: \g \in \partial f(\x)\} \leq \epsilon$.} the standard subgradient method fails to approach a Clarke stationary point of a Lipschitz function in general~\citep{Daniilidis-2020-Pathological}; moreover, it is not possible to find such points using any algorithm within finite time \citep[Theorem~1]{Zhang-2020-Complexity}. Moreover, even getting \textit{near} an approximate Clarke stationary point of a Lipschitz function has been proven to be impossible unless the number of queries has an exponential dependence on the dimension~\citep{Kornowski-2021-Oracle}. 
For an overview of relevant theoretical results in nonsmooth nonconvex optimization, we refer to Appendix \ref{sec:related_work}. 

These negative results motivate rethinking the definition of local optimality in terms of a relaxed yet still meaningful notion. To this end, we consider the problem of finding a $(\delta,\epsilon)$-Goldstein stationary point of $f$ \citep{Goldstein-1977-Optimization}, which are points for which there exists a convex combination of gradients in a $\delta$-neighborhood whose norm is less than $\epsilon$ (see \secref{sec:prelim} for a formal definition).
The breakthrough result of \citet{Zhang-2020-Complexity} proposed a randomized algorithm that finds such points with a dimension-free complexity of $\widetilde{O}(\delta^{-1}\epsilon^{-3})$ oracle calls.
Though they make use of a non-standard first-order oracle that does not apply to all Lipschitz functions, subsequent work~\citep{Davis-2022-Gradient, Tian-2022-Finite} has proposed variants of the algorithm that apply to any Lipschitz function using a standard first-order oracle.

It is important to note that all of the aforementioned algorithms are randomized. This state of affairs is unusual when contrasted with the regimes of smooth or convex optimization, where deterministic optimal dimension-free first-order algorithms exist and cannot be improved upon by randomized algorithms~\citep{Nesterov-2018-Lectures,Carmon-2021-Lower}. This raises a fundamental question:
\begin{quote}
\centering
\textit{
What is the role of randomization in dimension-free nonsmooth nonconvex optimization?} 
\end{quote}

\subsection{Our Contributions} \label{sec:contribution}

This paper presents several results on the complexity of finding $(\delta, \epsilon)$-Goldstein stationary points using deterministic algorithms, providing a detailed answer to the question raised above. Our contributions can be summarized as follows: 
\begin{enumerate}
\item \textbf{Necessity of randomness for dimension-free complexity (Theorem~\ref{Theorem:DET}).} 
We show that deterministic algorithms cannot find $(\delta, \epsilon)$-Goldstein stationary points at \emph{any} dimension-free rate, by proving a dimension-dependent lower bound of $\Omega(d)$ for any deterministic first-order algorithm, where $\delta, \epsilon > 0$ are smaller than given constants. 

\item \textbf{Deterministic algorithms require a zeroth-order oracle (Theorem~\ref{Theorem:1LB}).} In sharp contrast to smooth or convex optimization, we prove that without access to function values, no deterministic algorithm can guarantee to return a $(\delta,\epsilon)$-Goldstein stationary point within \emph{any} finite time, whenever $\delta, \epsilon > 0$ are smaller than given constants. On the other hand, we note that a gradient oracle is sufficient to obtain a finite-time guarantee using a randomized algorithm (Remark \ref{rem:randomized1stOracle}).

\item \textbf{Deterministic algorithm with logarithmic smoothness dependence (Theorem~\ref{Theorem:SMOOTH}).} Considering cases in which the objective function is slightly smooth, we present a deterministic first-order algorithm that finds a $(\delta,\epsilon)$-Goldstein stationary point of any $H$-smooth function within $\widetilde{O}(\log(H)\delta^{-1}\epsilon^{-3})$ oracle calls.\footnote{A function $f:\reals^d\to\reals$ is called $H$-smooth if for all $\x,\y\in\reals^d:\norm{\nabla f(\x)-\nabla f(\x)}\leq H\norm{\x-\y}$.}

\item \textbf{Deterministic smoothing (\thmref{thm:blackbox} and \thmref{thm:circuitSmoothing}).} We show that unlike randomized black-box smoothings, no deterministic black-box smoothing can produce a reasonable $\textnormal{poly}(d)$-smooth approximation using a dimension-free complexity, essentially solving an open question due to \citet{Kornowski-2021-Oracle}.
On the other hand, in a practical white-box model of ReLU neural networks, we propose a simple, dimension-free, deterministic smoothing procedure which applies to a variety of architectures, while provably maintaining the set of $(\delta,\epsilon)$-Goldstein stationary points. Combined with the algorithm described in the previous bullet, we obtain the first deterministic, dimension-free algorithm for optimizing ReLU networks, circumventing our aforementioned lower bound.
\end{enumerate}

\paragraph{Related work.} 
Following an initial publication of our results, \citet{Tian-2022-No} have independently presented an alternative proof of our first result (\thmref{Theorem:DET}).
A more detailed account of previous results in nonsmooth nonconvex optimization is deferred to Appendix \ref{sec:related_work}.

\section{Preliminaries and Technical Background}\label{sec:prelim}

\paragraph{Notation.} We denote $[d] := \{1, 2, \ldots, d\}$. We denote by $\zero_d \in \br^d$ the zero vector and by $\e_1, \e_2, \ldots, \e_d\in\reals^d$ the standard basis vectors.
For any vector $\x \in \br^d$, we let $\|\x\|$ be its Euclidean norm, and denote by $x_i$ its $i^\textnormal{th}$ coordinate.
For a set $\XCal \subseteq \br^d$, we let $\textnormal{conv}(\XCal)$ denote its convex hull. For a continuous function $f(\cdot): \br^d \mapsto \br$, we let $\nabla f(\x)$ denote the gradient of $f$ at $\x$ (if it exists). For a scalar $a \in \br$, we let $\lfloor a \rfloor$ and $\lceil a\rceil$ be the smallest integer that is larger than $a$ and the largest integer that is smaller than $a$. In addition, we denote a closed ball of radius $r>0$
around a point $\x \in \br^d$ by $B_{r}(\x) := \{\y \in \br^d: \|\y - \x\| \leq r\}$.
Given a bounded segment $I\subset \reals$, we denote by $\xi\sim U(I)$ a random variable distributed uniformly over $I$.
Finally, we use the standard big-O notation, with $O(\cdot)$, $\Theta(\cdot)$ and $\Omega(\cdot)$ hiding absolute constants that do not depend on problem parameters, $\tilde{O}(\cdot)$ and $\tilde{\Omega}(\cdot)$ hiding absolute constants and additional logarithmic factors, and also denote by $\poly(\cdot)$ polynomial factors.

\paragraph{Nonsmooth analysis.}
We call a function $f:\reals^d\to\reals$ $L$-Lipschitz if for any $\x,\y\in\reals^d:|f(\x)-f(\y)|\leq L\norm{\x-\y}$, and $H$-smooth if it is differentiable and $\nabla f:\reals^d\to\reals^d$ is $H$-Lipschitz, namely for any $\x,\y\in\reals^d:\norm{\nabla f(\x)-\nabla f(\y)}\leq H\norm{\x-\y}$.
By Rademacher's theorem, Lipschitz functions are differentiable almost everywhere (in the sense of Lebesgue). Hence, for any Lipschitz function $f:\reals^d\to\reals$ and point $\x\in\reals^d$ the Clarke subgradient set \citep{Clarke-1990-Optimization} can be defined as
\[
\partial f(\x):=\conv\{\g\,:\,\g=\lim_{n\to\infty}\nabla f(\x_n),\,\x_n\to \x\}~,
\]
namely, the convex hull of all limit points of $\nabla f(\x_n)$ over all sequences of differentiable points which converge to $\x$.
Note that if the function is continuously differentiable at a point or convex, the Clarke subdifferential reduces to the gradient or subgradient in the convex analytic sense, respectively.
We say that a point $\x$ is an $\epsilon$-Clarke stationary point of $f(\cdot)$ if $\min\{\norm{\g}:\g\in\partial f(\x)\}\leq\epsilon$.
Furthermore, given $\delta>0$ the Goldstein $\delta$-subdifferential \citep{Goldstein-1977-Optimization} of $f$ at $\x$ is the set
\[
\partial_{\delta}f(\x):=\conv\left(\cup_{\y\in B_{\delta}(\x)}\partial f(\y)\right)~,
\]
namely all convex combinations of gradients at points in a $\delta$-neighborhood of $\x$. We say that a point $\x$ is a $(\delta,\epsilon)$-Goldstein stationary point of $f(\cdot)$ if
\[
\min\{\norm{\g}:\g\in\partial_{\delta} f(\x)\}\leq\epsilon~.
\]
Note that a point is $\epsilon$-Clarke stationary if and only if it is $(\delta,\epsilon)$-Goldstein stationary for all $\delta>0$ \citep[Lemma 7]{Zhang-2020-Complexity}.

\paragraph{Algorithms and complexity.}
Throughout this work we consider iterative first-order algorithms, from an oracle complexity perspective \citep{Nemirovski-1983-Problem}. Such an algorithm first produces an initial point $\x_0\in\reals^d$ (possibly at random, if it is a randomized algorithm) and receives $(f(\x_0),\partial f(\x_0))$.\footnote{For the purpose of this work it makes no difference whether the algorithm gets to see some subgradient or the whole Clarke subgradient set. That is, the lower bounds to follow hold even if the algorithm has access to the \emph{entire} subgradient set, while the upper bounds hold even if the algorithm receives a single arbitrary subgradient.} Then, for any $t\geq1$ produces $\x_t$ possibly at random based on previously observed responses, and receives $(f(\x_t),\partial f(\x_t))$. We are interested in the minimal number $T$ for which we can guarantee to produce some $(\delta,\epsilon)$-Goldstein stationary point, uniformly over the class of Lipschitz functions.

\section{Lower bounds for deterministic algorithms}\label{sec:lower}

\subsection{Dimension-dependent lower bound}

As discussed earlier, \citep{Zhang-2020-Complexity,Davis-2022-Gradient,Tian-2022-Finite} have presented randomized first-order algorithms that given any $L$-Lipschitz function $f:\reals^d\to\reals$ and an initial point $\x_0$ that satisfies $f(\x_0)-\inf_{\x}f(\x)\leq \Delta$, produce a $(\delta,\epsilon)$-Goldstein stationary point of $f$ within $\widetilde{O}(\Delta L^2/\delta\epsilon^3)$ oracle calls to $f$. We show that this rate, and indeed any dimension-free rate, cannot be achieved by a deterministic algorithm.

\begin{theorem} \label{Theorem:DET}
For any $\Delta,L>0,~d\geq3$, any $T\leq d-2$ and any deterministic first-order algorithm, there exists an $L$-Lipschitz function $f:\reals^d\to\reals$ such that $f(\x_0)-\inf_{\x}f(\x)\leq \Delta$, yet the first $T$ iterates produced by the algorithm when applied to $f$ are not $(\delta,\epsilon)$-stationary points for any $\delta<\frac{\Delta}{L},\,\epsilon<\frac{L}{252}$.
\end{theorem}

Our result highlights that even though finding a $(\delta, \epsilon)$-Goldstein stationary point in nonsmooth nonconvex optimization is computationally tractable using a randomized algorithm, it is essentially harder than finding an $\epsilon$-stationary point in smooth nonconvex optimization without randomization as it requires $\Omega(d)$ oracle calls. 
We also note that \thmref{Theorem:DET} holds true regardless of the relationship between the dimension $d$ and the parameters $(\delta, \epsilon)$, 
in contrast to the dimension-independent lower bounds established for nonsmooth convex optimization~\citep{Nesterov-2018-Lectures}, where the accuracy parameter must scale polynomially with $1/d$.

The full proof of \thmref{Theorem:DET} is deferred to Section \ref{d_lower_proof}, though we will now provide a proof sketch. For any deterministic first-order algorithm, if an oracle can always return the ``uninformative'' answer $f(\x_t)=0,\nabla f(\x_t)=\e_1$ this fixes the iterates $\x_1,\dots,\x_T$. Hence, it remains to construct a Lipschitz function that will be consistent with the oracle answers, yet all the queried points are not $(\delta,\epsilon)$-stationary. To that end, we construct a function which in a very small neighborhood of each queried point $\x_t$ looks like $\x\mapsto \e_1^\top(\x-\x_t)$, yet in most of the space looks like $\x\mapsto\max\{\sv^\top \x,-1\}$, which has $(\delta,\epsilon)$-stationary points only when $\x$ is correlated with $-\sv$. By letting $\sv$ be some vector which is orthogonal to all the queried points (which is possible as long as $T<d-1$), we obtain the result.

This construction relies crucially on the function being highly nonsmooth---essentially interpolating between two orthogonal linear functions in an arbitrarily small neighborhood.
As it will turn out, if the function to be optimized is even slightly smooth, then the theorem can be bypassed, as we will show in \secref{sec:alg}.

\subsection{Lower bound for gradient-only oracle}\label{sec:grad}

In this section, we demonstrate the importance of having access either to randomness or to a zeroth-order oracle, namely to the function value. 
In particular, we prove that any deterministic algorithm which has access only to a gradient oracle cannot return an approximate Goldstein stationary point within any finite number of iterations.

\begin{theorem} \label{Theorem:1LB}
For any $0 < \delta < \epsilon < 1$, any $d\in\NN,~T<\infty$, and any deterministic algorithm which has access only to a gradient oracle, there exists a $1$-Lipschitz function $f : \br^d \to [-1, 1]$ such that the algorithm cannot guarantee to return a $(\delta, \epsilon)$-Goldstein stationary point using $T$ oracle calls.
\end{theorem}

We will now sketch the proof; see Section \ref{inf_lower_proof} for the full proof.
We can assume without loss of generality that $d=1$ (otherwise we can simple apply the ``hard'' construction to the first coordinate).
Suppose a deterministic algorithm has access only to a derivative oracle, which always returns the ``uninformative'' answer $f'(\x_t)=1$. This fixes the algorithm's iterates $\x_1,\dots,\x_T$, which then attempts to guarantee that some returned point $\hat{\x}$ is a $(\delta,\epsilon)$-stationary point. It remains to construct a Lipschitz function that will be consistent with the oracle answers, yet $\hat{\x}$ will not be $(\delta,\epsilon)$-stationary. To that end, we construct a function which looks like $\x\mapsto \x-\hat{\x}$ in a long enough segment around $\hat{\x}$, ensuring it is indeed not $(\delta,\epsilon)$-stationary. On the other hand,
in a very small neighborhood of each queried point $\x_t$ we add a ``bump'' so that the function looks like $\x\mapsto \x-\x_t$, consistent with our resisting oracle. Finally, far away from all queried points we let the function be constant, so that its image remains in $[-1,1]$.

\begin{remark} \label{rem:randomized1stOracle}
In contrast with \thmref{Theorem:1LB}, there exist randomized algorithms that access only a gradient oracle and guarantee to return a $(\delta,\epsilon)$-Goldstein stationary point in finite-time.
Indeed,~\citet[Theorem~3.1]{Lin-2022-Gradient} have shown that for $f_\delta(\x) := \EE_{\su \sim B_\delta(\x)}[f(\su)]$ it holds that
$\nabla f_\delta(\x) = \EE_{\su \sim B_\delta(\x)}[\nabla f(\su)] \in \partial_\delta f(\x)$ (where $\su \sim B_\delta(\x)$ is distributed uniformly over a Euclidean ball of radius $\delta$ centered at $\x$), thus it suffices to find an $\epsilon$-stationary point of $f_\delta$. But since $\nabla f(\su)$ is an unbiased estimator of $\nabla f_\delta(\x)$ and $\|\nabla f(\su)\| \leq L$, this is well known to be possible using stochastic gradient descent \citep{ghadimi2013stochastic}. In particular, the same argument as in the proof of \citet[Theorem~3.2]{Lin-2022-Gradient} shows that it is possible to find such a point within $O(\sqrt{d}(L^4\epsilon^{-4} + \Delta L^3\delta^{-1}\epsilon^{-4}))$ calls to a gradient oracle.
\end{remark}

\begin{remark}
    Note that in nonsmooth convex optimization or in smooth nonconvex optimization, a deterministic algorithm can obtain $(\delta,\epsilon)$-Goldstein stationary points using only a gradient oracle, even at a dimension-free rate. Indeed, in the nonsmooth convex case gradient descent returns $\x$ such that $f(\x)-\inf_{\x}f(\x)<\delta\epsilon$ within $O(\delta^{-2}\epsilon^{-2})$ gradient evaluations, and any such point is in particular a $(\delta,\epsilon)$-Goldstein stationary point.\footnote{Otherwise, let $\g$ be the minimal norm element in $\partial_{\delta}f(\x)$, and assume by contradiction that $\norm{\g}>\epsilon$. Then \citet{Goldstein-1977-Optimization} ensures that $f(\x-\frac{\delta}{\norm{\g}}\g)\leq f(\x)-\delta\norm{\g}<f(\x)-\delta\epsilon<\inf_{\x}f(\x)$ which is a contradiction.} Similarly, in the smooth nonconvex setting gradient descent returns an $\epsilon$-stationary point within $O(\epsilon^{-2})$ gradient evaluations, which is trivially also a $(\delta,\epsilon)$-Goldstein stationary point.
\end{remark}

\section{Deterministic algorithm for slightly smooth functions} \label{sec:alg}

In this section we show that if the objective function is even slightly smooth, then the dimension free rate of $\widetilde{O}(\delta^{-1}\epsilon^{-3})$ can be obtained by a deterministic first-order algorithm, incurring a mild logarithmic dependence on the smoothness parameter.

\begin{theorem}\label{Theorem:SMOOTH}
Suppose $f:\reals^d\to\reals$ is $L$-Lipschitz, $H$-smooth, and $\x_0\in\reals^d$ is such that $f(\x_0)-\inf_{\x}f(\x)\leq\Delta$. Then \textsc{Deterministic-Goldstein-SG}$(\x_0,\delta,\epsilon)$ (Algorithm \ref{Alg:Full}) is a deterministic first-order algorithm that given any $\delta,\epsilon\in(0,1)$ returns a $(\delta, \epsilon)$-Goldstein stationary point of $f$ within $T=O\left(\frac{\Delta L^2\log(HL\delta/\epsilon)}{\delta\epsilon^3}\right)$ oracle calls.
\end{theorem}

The idea behind this de-randomization is to replace a certain randomized line search in the algorithms of \citet{Zhang-2020-Complexity,Davis-2022-Gradient,Tian-2022-Finite}, which in turn are based on \citet{Goldstein-1977-Optimization},
with a deterministic binary search subroutine, Algorithm \ref{Alg:BS}.
This subroutine terminates within ${O}(\log(H\delta/\epsilon))$ steps provided that the function is $H$-smooth. We note that such a procedure was derived by \citet{Davis-2022-Gradient} for any $H$-weakly convex function along differentiable directions, and since any $H$-smooth function is $H$-weakly convex and differentiable along any direction, this can be applied in an identical manner.
Although this algorithmic ingredient appears inside a proof of \citet{Davis-2022-Gradient}, they use it in a different manner in order to produce a \emph{randomized} algorithm for weakly convex functions in low dimension, with different guarantees suitable for that setting.
We defer the full analyses of Algorithm \ref{Alg:BS}, Algorithm \ref{Alg:Full} which lead to the proof of
\thmref{Theorem:SMOOTH} to Section \ref{alg_proof}.

\begin{algorithm}[!t]
\begin{algorithmic}\caption{\textsc{Binary-Search}($\delta$, $\nabla f(\cdot)$, $\g_0$, $\x$)}\label{Alg:BS}
\STATE \textbf{Initialization:} Set $b \leftarrow \delta$, $a \leftarrow 0$ and $t \leftarrow b$.
\WHILE{$-\nabla f(\x - t\tfrac{\g_0}{\|\g_0\|})\cdot\tfrac{\g_0}{\|\g_0\|} + \tfrac{1}{2}\|\g_0\|\geq -\frac{\epsilon}{4}$}
\STATE Set $t \leftarrow \frac{a+b}{2}$.  
\IF{$f(\x - b\tfrac{\g_0}{\|\g_0\|})  + \tfrac{b}{2}\|\g_0\| > f(\x - t\tfrac{\g_0}{\|\g_0\|})  + \tfrac{t}{2}\|\g_0\|$}
\STATE Set $a \leftarrow t$.  
\ELSE
\STATE Set $b \leftarrow t$. 
\ENDIF
\ENDWHILE
\STATE \textbf{Output:} $\nabla f(\x - t\frac{\g_0}{\|\g_0\|})$. 
\end{algorithmic}
\end{algorithm}

\begin{algorithm}[!t]
\begin{algorithmic}[1]\caption{\textsc{Deterministic-Goldstein-SG}$(\x_0,\delta,\epsilon)$}\label{Alg:Full}
\STATE \textbf{Input:} initial point $\x_0 \in \br^d$, accuracy parameters $\delta, \epsilon \in (0, 1)$. 
\FOR{$t = 0, 1, 2, \ldots, T-1$}
\STATE Set $\g(\x_t)\leftarrow\nabla f(\x_t)$.
\WHILE{$f(\x_t - \delta\tfrac{\g(\x_t)}{\|\g(\x_t)\|}) - f(\x_t) > - \tfrac{\delta}{2}\|\g(\x_t)\|$ and $\|\g(\x_t)\| > \epsilon$}
\STATE Set $\g_{\textnormal{new}} \leftarrow \textsc{Binary-Search}(\delta, \nabla f(\cdot), \g(\x_t), \x_t)$. 
\STATE $\h_t\leftarrow\arg\min\{\norm{\g(\x_t)+\lambda(\g_{\textnormal{new}}-\g(\x_t))}:0\leq\lambda\leq1\}$.
\STATE $\g(\x_t)\leftarrow\h_t$.
\ENDWHILE
\IF{$\|\g(\x_t)\| \leq \epsilon$}
\STATE \textbf{Stop}.
\ELSE
\STATE $\x_{t+1} \leftarrow \x_t - \delta\frac{\g(\x_t)}{\|\g(\x_t)\|}$. 
\ENDIF
\ENDFOR
\STATE \textbf{Output:} $\x_t$. 
\end{algorithmic}
\end{algorithm}

\section{Deterministic smoothings}\label{sec:smoothing}

Motivated by the mild smoothness dependence of \textsc{Deterministic-Goldstein-SG}
(Algorithm \ref{Alg:Full}) as proved in \thmref{Theorem:SMOOTH}, we turn to the design of smoothing procedures. These are algorithms that act on a Lipschitz function, and return a smooth approximation---allowing the use of smooth optimization methods. Smoothing nonsmooth functions in order to allow the use of smooth optimization algorithms is a longstanding approach for nonsmooth nonconvex optimization, both in practice and in theoretical analyses. We refer to Appendix \ref{sec:related_work} for references.
From a computational perspective, it is not clear what it means for an algorithm to ``receive'' a real function as an input. For this reason, we make the distinction between ``black-box'' smoothings which are granted oracle access to the original function, and ``white-box'' smoothings which are assumed to have access to additional structural information.

\subsection{Black-box smoothings}

Recently, \citet{Kornowski-2021-Oracle} have studied black-box smoothings from an oracle complexity viewpoint. One of their main results is that randomized smoothing \citep{Duchi-2012-Randomized} is an optimal smoothing procedure, in the sense that no efficient black-box smoothing procedure can yield an approximation whose smoothness parameter is lower than $O(\sqrt{d})$, which is achieved by randomized smoothing. In particular, this implies that any efficient black-box smoothing unavoidably suffers from some dimension dependence.
In that paper, the authors posed the open question of assessing what can be achieved by a deterministic black-box smoothing, since efficient randomized smoothing is only able to return stochastic estimates of the smoothed function. We solve this question for all ``reasonable'' smoothing procedures, as defined next. Without loss of generality, we consider functions whose Lipschitz constant is 1, since 
if the objective function is $L$-Lipschitz the algorithm can simply rescale it by $L$.

\begin{definition}
An algorithm $\Scal$ is called a black-box smoothing with complexity $T\in\NN$, if it uses a first-order oracle of a 1-Lipschitz function $f:\reals^d\to\reals$, such that given any $\x\in\reals^d$ it sequentially queries $f$'s oracle at $T$ points and returns $\widetilde{f}(\x),\,\g_\x=\nabla \widetilde{f}(\x)$ for some smooth $\widetilde{f}:\reals^d\to\reals$.
We say that the smoothing algorithm is \textbf{meaningful} if $\widetilde{f}$ is $\poly(d)$-smooth, and any $(\delta,\epsilon)$-Goldstein stationary point of $\widetilde{f}$ is a $(\poly(\delta,\epsilon),\poly(\delta,\epsilon))$-Goldstein stationary point of $f$.
\end{definition}

In other words, a smoothing fails to be meaningful if either the smooth approximation has super-polynomial smoothness (thus can hardly be treated as smooth), or introduces completely ``fake'' approximately-stationary points of $f$.\footnote{It is important to recall that for smooth $\widetilde{f}$ the notions of approximate-Clarke stationarity and approximate-Goldstein stationarity coincide \citep[Proposition 6]{Zhang-2020-Complexity}.}
The latter case implies that running a nonconvex optimization algorithm over $\widetilde{f}$ fails to provide any meaningful guarantee for the original function $f$. Note that these assumptions are extremely permissive, as we allow for any polynomial parameter blow-up, and do not even quantify the requirement regarding the accuracy of the approximation.
Notably, all black-box smoothings considered in the literature, including randomized smoothing and the Moreau-Yosida smoothing for weakly-convex functions \citep{davis2019stochastic}, are easily verified to be meaningful (and, indeed, satisfy more stringent conditions with respect to the original function). Further note that the randomized complexity of these procedures is dimension-free.

Under this mild assumption, our previous theorems readily imply a answer to the question posed by \citet{Kornowski-2021-Oracle}.

\begin{theorem} \label{thm:blackbox}
There is no deterministic, black-box meaningful smoothing algorithm with dimension-free complexity.
\end{theorem}

\begin{proof}
Assuming towards contradiction there is such a smoothing algorithm $\Scal$, we compose it with \textsc{Deterministic-Goldstein-SG}. Namely, given any Lipschitz $f$, we consider the first-order algorithm obtained by applying Algorithm \ref{Alg:Full} to $\widetilde{f}=\Scal(f)$. Since $\widetilde{f}$ is $\poly(d)$-smooth, and any $(\delta,\epsilon)$-Goldstein stationary point of $\widetilde{f}$ is a $(\poly(\delta,\epsilon),\poly(\delta,\epsilon))$-Goldstein stationary point of $f$, by \thmref{Theorem:SMOOTH} we obtain overall a deterministic algorithm that finds a $(\poly(\delta,\epsilon),\poly(\delta,\epsilon))$-Goldstein stationary point of $f$ within $O(\log(\poly(d))\cdot\poly(\delta^{-1},\epsilon^{-1}))=O(\log(d)\cdot\poly(\delta^{-1},\epsilon^{-1}))$ first-order oracle calls---contradicting \thmref{Theorem:DET}.
\end{proof}

\subsection{Deterministic smoothing of ReLU networks}

In this section we introduce a smoothing technique that can be applied to optimization of non-smooth functions, provided that they are expressed as ReLUs in a neural network accessible
to the smoothing procedure. The idea of utilizing the representation of a function, as opposed to just having oracle access to it, has been commonly used across diverse domains, from purely theoretical applications~\citealp[e.g., computational complexity theory,][]{Daskalakis-2011-Continuous, Fearnley-2021-Complexity} to practical applications~\citealp[e.g., deep neural networks,][]{Lecun-2015-Deep, Goodfellow-2016-Deep}. We refer to this function representation as the \textit{white-box model} to contrast it with the previously discussed black-box model. Our results demonstrate that having such a white box access is powerful enough to allow for meaningful deterministic smoothing, as opposed to the black-box model whose insufficiency is established in \thmref{thm:blackbox}.

We start by giving a brief overview of the key observation underlying our deterministic smoothing approach. Consider a single ReLU neuron with a bias term, namely for some point $\x\in\reals^{d}$, weight $\w\in\reals^{d}$ and bias $b\in\reals:$
\begin{equation} \label{eq:relu}
(\x,\w,b)\mapsto \relu(\w^{\top}\x+b)
:=\max\{\w^{\top}\x+b,0\}~.
\end{equation}
We replace the nonsmooth ReLU with a smooth, carefully chosen ``Huberized'' function:
\[
\softrelu_{\gamma}(z)
=\mathbb{E}_{\xi \sim U[- \gamma, \gamma]}[\relu(z + \xi)]
=\begin{cases}
z~, & z\geq \gamma
\\
\frac{(z+a)^2}{4a}~, & -\gamma\leq z< \gamma
\\
0~,& z< -\gamma,
\end{cases}
\]
for some small $\gamma>0$.
Accordingly, we obtain the ``smoothed'' neuron of the form
\begin{align*}
(\x,\w,b)\mapsto &~\softrelu_a(\w^{\top}\x+b)
\\
&=\mathbb{E}_{\xi \sim U[- \gamma, \gamma]}[\relu(\w^{\top}\x+(b + \xi))]
=\begin{cases}
\w^{\top}\x+b~, & \w^{\top}\x+b\geq \gamma
\\
\frac{(\w^{\top}\x+b+a)^2}{4a}~, & -\gamma\leq \w^{\top}\x+b< \gamma
\\
0~,& \w^{\top}\x+b< -\gamma
\end{cases}
~.
\end{align*}
Optimizing the function above (as a component of a larger neural network) with respect to $(\w,b)$ is the goal of any optimizer seeking to ``train'' the network's parameters to fit its input $\x$. We see that on one hand the smoothed neuron is a closed-form smooth approximation of the ReLU neuron in \eqref{eq:relu},
yet is mathematically equivalent to randomized smoothing \emph{over the bias term}. Hence, we obtain the meaningful guarantees of randomized smoothing, namely that optimizing the smoothed model corresponds to optimizing the original nonsmooth function, without the need for randomization. Moreover, as opposed to
plain randomized smoothing which would smooth with respect to $(\w,b)\in\reals^{d+1}$, thus suffering from a dimension dependence in the smoothness parameter, smoothing over $b$ alone avoids dependence on $d$. Overall, replacing all ReLU neurons of a network with smoothed neurons is mathematically equivalent to randomized smoothing over the parameter subspace corresponding to all bias terms, reducing the dimension-dependence to a dependence on the
number of biases, roughly the size of the network.\footnote{Note that due to dependencies between neurons at different layers, this does not correspond to standard randomized smoothing with respect to an isotropic distribution, but rather to a nontrivial distribution capturing the dependencies among different bias terms. We remark that this is a major technical challenge in proving \thmref{thm:circuitSmoothing} to follow.}

We now formally describe the class of representations that our smoothing procedure will apply to. It is easy to see that this class contains ReLU neural networks with biases of arbitrary depth and width, including many architectures used in practice.

\begin{definition}[Neural Arithmetic Circuits (NAC)]\label{def:AC}
We say that $\calC$ is a \emph{neural arithmetic circuit with biases} if it is represented as a directed acyclic graph with four different group of nodes: (i) input nodes; (ii) bias nodes; (iii) output nodes; and (iv) gate nodes. The gate node can be one of $\{+, \relu, \times, \mathrm{const(c)}\}$, where $\mathrm{const(c)}$ stands for a constant $c \in [-1, 1]$. Moreover, a valid NAC with biases satisfies the following conditions: 
\begin{enumerate}
\item There is at least one input node. Every input node has $0$ incoming edges and any number of outgoing edges.
\item The number of bias nodes is equal to the number of $\relu$ gates. Every bias node has $0$ incoming edges, and only one outgoing edge. 
\item The gate nodes in $\{+, \times\}$ have two incoming edges, and any number of outgoing edges.\footnote{We can generalize it to the case of any finite number of inputs. Focusing on two incoming edges does not lack the generality since we can always compose these gates to simulate addition and maximum with many inputs by just increasing the size and the depth of the circuit by a logarithmic factor.}
\item The gate node $\mathrm{const(c)}$ has 0 incoming edges, and any number of outgoing edges. 
\item The gate node $\relu$ has 1 incoming edge
but any number of outgoing edges. We also assume that all the $\relu$ gates have biases, i.e., the predecessor vertex of a $\relu$ gate is always a ``$+$'' gate connected to a bias node that is unique for every $\relu$ gate.
\item There is only one output node that has 1 incoming edge and 0 outgoing edges.
\end{enumerate}
We denote by $s(\calC)$ the \textit{size} of $\calC$ (i.e., the number of nodes in the graph of $\calC$).
\end{definition}

The interpretation of $\calC$ as a function $f : \br^d \to \br$ is very intuitive.
The input nodes correspond to the input variables $x_1, \ldots, x_d$, followed by a gate node defining arithmetic operations over their input, finally producing $f(\x)$ in the output node.

\begin{example} \label{eg:neuralCircuits1}
  Consider training a neural network $\Phi_{\W, \bb}$ to fit a labeled dataset $(\x_i, y_i)_{i=1}^{n}$ with respect to the quadratic loss, where $\W,\bb$ are the vectors of weights and biases of $\Phi$, respectively. This task corresponds to minimizing the following function:
  \[ f(\W, \bb) = \sum_{i=1}^{n} \left(\Phi_{\W, \bb}(\x_i) - y_i\right)^2. \]
  It is easy to see that this function can be expressed as a neural arithmetic circuit according to Definition \ref{def:AC}. The only requirement for $\Phi_{\W, \bb}$ is that every $\relu$ gate has a unique bias variable. Examples for such $\Phi_{\W, \bb}$ include feed-forward ReLU networks, convolutional networks, and residual neural network with skip connections.
\end{example}

Following the example above, we see that the problem of finding a $(\delta, \epsilon)$-Goldstein stationary point of a function represented by a neural arithmetic circuit $\calC$
captures a wide range of important nonsmooth and nonconvex problems.
To prove the efficiency of our proposed method, we need to impose the following assumption, measuring the extent to which function values increase throughout the neural arithmetic circuit.
We note in Remark \ref{rem:designNNs} that this assumption is satisfied by the practical design of deep neural networks.

\begin{assumption} \label{asp:recursiveLipschitz}
For $G>0$, we say that $h:\reals^d\to\reals$ is $G$-bounded over $\calR$ if $|h|_\calR|\leq G$.
Suppose $f : \br^d \to \br$  is represented as a linear arithmetic circuit $\calC$. Let $v_1, \dots, v_n$ be the nodes in $\calC$ and $f_i$ be the function that will be computed if $v_i$ would the output of the neural circuit. We assume that there is a set $\calR \subseteq \br^d$, such that for all $i\in[n]:f_i$ is $L_i$-Lipschitz and $G_i$-bounded over $\calR$,
according to the following composition rules:
\begin{enumerate}
\item[-] \textbf{$\boldsymbol{v_i}$ is a $\boldsymbol{+}$ gate:} if $f_i = f_j + f_k$ then $L_i = L_j + L_k$ and $G_i = G_j + G_k$. 
\item[-] \textbf{$\boldsymbol{v_i}$ is a $\boldsymbol{\relu}$ gate:} if $f_i = \relu\{f_j\}$ then $L_i = L_j$ and $G_i = G_j$. 
\item[-] \textbf{$\boldsymbol{v_i}$ is a $\boldsymbol{\mathrm{const}(c)}$ gate:} $L_i = 0$ and $G_i = c$. 
\item[-] \textbf{$\boldsymbol{v_i}$ is a $\boldsymbol{\times}$ gate:} if $f_i = f_j \cdot f_k$ then $L_i = L_j \cdot G_k + G_j \cdot L_k$ and $G_i = G_j \cdot G_k$. 
\item[-] \textbf{$\boldsymbol{v_i}$ is a input or a bias node:} $L_i = 1$, $G_i = \mathrm{diam}(\calR)$ (the diameter of $\calR$).
\end{enumerate}
In particular, we assume that $f$ is $L$-Lipschitz and $G$ bounded over $\calR$ according to the rules above. In this case, we say that $f$ is \textit{$L$-recursively Lipschitz} and \textit{$G$-recursively bounded} in $\calR$.
\end{assumption}
\begin{remark}\label{rem:designNNs}
Note that the recursive rules used in Assumption \ref{asp:recursiveLipschitz} always provide an upper bound on $L > 0$, however this bound can be much larger than the true Lipschitz constant $L$ in the worst-case. To bypass these bad cases, we impose Assumption \ref{asp:recursiveLipschitz}.
Notably, this assumption is not theoretically artificial but is satisfied by generic constructions of neural networks in the context of deep learning. Indeed, since the $\boldsymbol{+}$ and $\boldsymbol{\times}$ gates are often used consecutively, leading to a bad Lipschitz constant in the worst case,
practitioners often force these upper bounds to be as small as possible by employing normalization techniques in order to stabilize the training
~\citep{Ioffe-2015-Batch, Miyato-2018-Spectral}. 
\end{remark}

As previously discussed, our deterministic smoothing idea is to replace the $\relu$ activation function with its carefully chosen smooth alternative $\softrelu$. We emphasize that this smoothing procedure is simple, implementable and inspired by techniques that are widely accepted in practice (e.g \citealp{tatro2020optimizing,shamir2020smooth}). While proving that this results in a smooth approximation of the original function is relatively straightforward, the main novelty of our proof is showing that any $(\delta,\epsilon)$-Goldstein stationary point of the smoothed model is a $(\delta,\epsilon)$-Goldstein stationary point of the original, following from our observation of the equivalence to randomized smoothing with respect to a low dimensional subspace.
This is crucial, as it allows optimization of the original function to be carried through the smoothed model. We are now ready to state our main theorem in this section, whose proof is deferred to Section \ref{app:circuitSmoothing}.

\begin{theorem}\label{thm:circuitSmoothing}
Let $f : \br^d \to \mathbb{R}$ be a $L$-recursively Lipschitz and $G$-recursively bounded function in $\calR \subseteq \br^d$ (see Assumption \ref{asp:recursiveLipschitz}), represented by a neural arithmetic circuit $\calC$. For every $\gamma > 0$, we can construct a function $\widetilde{f} : \br^d \to \br$ such that for all $\x \in \calR$ it holds that:
\begin{enumerate}
\item $|f(\x) - \widetilde{f}(\x)| \le \gamma$. 
\item $\widetilde{f}$ is $L$-Lipschitz and $G$-bounded. 
\item $\widetilde{f}$ is $\frac{(G \cdot L)^{O(s(\calC))}}{\min\{\epsilon, \delta, \gamma\}}$-smooth.
\item Every $(\delta, \epsilon)$-Goldstein stationary point of $\widetilde{f}$ is a $(\delta', \epsilon')$-Goldstein stationary point of $f$ with $\epsilon' = 2 \epsilon$ and $\delta' = 2 \delta$.
\end{enumerate}
\end{theorem}

At first glance, the smoothness parameter provided by the theorem above, though dimension-independent, may seem overwhelming as it depends exponentially on the size of the network. Luckily, this brings us back to \thmref{Theorem:SMOOTH} where we have proved that it is possible to incur merely a logarithmic dependence on this parameter, resulting in the following corollary by setting $\gamma = \min\{\epsilon, \delta\}$.

\begin{corollary} \label{cor:circuitSmoothing}
  Let $f : \br^d \to \mathbb{R}$ be a $L$-recursively Lipschitz and $G$-recursively bounded function in $\calR \subseteq \br^d$ (see Assumption \ref{asp:recursiveLipschitz}), represented by a neural arithmetic circuit $\calC$. Then if we apply \textsc{Deterministic-Goldstein-SG} (Algorithm \ref{Alg:Full}) to the function $\widetilde{f}$ defined in Theorem \ref{thm:circuitSmoothing} and $\mathcal{R}$ is such that the algorithm's iterates do not escape $\mathcal{R}$, the algorithm is guaranteed to return a $(\delta,\epsilon)$-Goldstein stationary point of $f$ using $O\left(\frac{G L^2 s(\calC) \log(GL\delta/\epsilon)}{\delta \epsilon^3}\right)$ first-order oracle calls.
\end{corollary}

\section{Proofs} \label{sec:proofs}

\subsection{Proof of \thmref{Theorem:DET}} \label{d_lower_proof}

Fix $d\geq 3,~\Delta,L>0$ and let $T\leq d-2$.
Consider the case that for any $t\in[T-1]$ the first-order oracle response is $f(\x_t)=0,\nabla f(\x_t)=\e_1$. Since the algorithm is deterministic this fixes the iterate sequence $\x_1,\dots,\x_T$.
We will show this resisting strategy is indeed consistent with a function which satisfies the conditions in the theorem.

To that end, we denote $r:=\min_{1\leq i\neq j\leq T}\norm{\x_i-\x_j}/4>0$ (without loss of generality) and fix some $\sv\in(\mathrm{span}\{\e_1,\x_1,\dots,\x_T\})^\perp$ with $\norm{\sv}=1$ (which exists since $d\geq T+2$). For any $\z\in\reals^d$ we define
\[
g_{\z}(\x):=\min\{\norm{\x-\z}^2/r^2,1\}\sv^\top \x+(1-\min\{\norm{\x-\z}^2/r^2,1\})\e_1^\top(\x-\z)~,
\]
and further define
\[
h(\x):=\begin{cases}
\sv^\top \x\,,&\forall t\in[T]:\norm{\x-\x_t}\geq r
\\
g_{\x_t}(\x)\,,&\exists t\in[T]:\norm{\x-\x_t}< r
\end{cases}
~.
\]
Note that $h$ is well defined since by definition of $r$ there cannot be $i\neq j$ such that $\norm{\x-\x_i}<r$ and $\norm{\x-\x_j}<r$.

\begin{lemma}
$h:\reals^d\to\reals$ as defined above is $7$-Lipschitz, satisfies for any $t\in[T]:\,h(\x_t)=0,\nabla h(\x_t)=\e_1$ and has no $(\delta,\frac{1}{36})$-stationary points for any $\delta>0$.
\end{lemma}

\begin{proof}
We start by noting that $h$ is continuous, since for any $\z$ and $(\y_n)_{n=1}^{\infty}\subset B_r(\z),\,\y_n\overset{n\to\infty}{\longrightarrow}\y$ such that $\norm{\y-\z}=r$ we have
\begin{align*}
\lim_{n\to\infty}h(\y_n)
&=\lim_{n\to\infty}g_{\z}(\y_n)
\\&=\lim_{n\to\infty}\left(\min\{\norm{\y_n-\z}^2/r^2,1\}\sv^\top \y_n+(1-\min\{\norm{\y_n-\z}^2/r^2,1\})\e_1^\top(\y_n-\z)\right)
\\&=\lim_{n\to\infty}\left(\frac{\norm{\y_n-\z}^2}{r^2}\cdot \sv^\top \y_n+\left(1-\frac{\norm{\y_n-\z}^2}{r^2}\right)\e_1^\top(\y_n-\z)\right)
\\&=\sv^\top \y
~.
\end{align*}
Having established continuity, since $\x\mapsto \sv^{\top}\x$ is clearly $1$-Lipschitz (in particular $7$-Lipschitz), in order to prove Lipschitzness of $h$ it is enough to show that $g_{\x_t}(\x)$ is $7$-Lipschitz in $\norm{\x-\x_t}<r$ for any $\x_t$. For any such $\x,\x_t$ we have
\begin{align}
\nabla g_{\x_t}(\x)
&~=~\frac{2\sv^\top \x}{r^2}(\x-\x_t)+\frac{\norm{\x-\x_t}^2}{r^2}\sv
-\frac{2\e_1^\top(\x-\x_t)}{r^2}(\x-\x_t)
-\frac{\norm{\x-\x_t}^2}{r^2}\e_1
+\e_1 \nonumber
\\
&\overset{\sv\perp \x_t}{=}~
\frac{2\sv^\top (\x-\x_t)}{r^2}(\x-\x_t)+\frac{\norm{\x-\x_t}^2}{r^2}\sv
-\frac{2\e_1^\top(\x-\x_t)}{r^2}(\x-\x_t)
-\frac{\norm{\x-\x_t}^2}{r^2}\e_1
+\e_1
~,
\label{eq: nabla g}
\end{align}
hence
\begin{align*}
\norm{\nabla g_{\x_t}(\x)}
&=\left\|\frac{2\sv^\top (\x-\x_t)}{r^2}(\x-\x_t)+\frac{\norm{\x-\x_t}^2}{r^2}\sv-\frac{2\e_1^\top(\x-\x_t)}{r^2}(\x-\x_t)
-\frac{\norm{\x-\x_t}^2}{r^2}\e_1+\e_1\right\|
\\
&\leq\frac{2\norm{\sv}\cdot \norm{\x-\x_t}^2}{r^2}+\frac{\norm{\x-\x_t}^2}{r^2}\norm{\sv}
+\frac{2\norm{\e_1}\cdot\norm{\x-\x_t}^2}{r^2}
+\frac{\norm{\x-\x_t}^2}{r^2}\norm{\e_1}
+\norm{\e_1}
\\
&\leq2+1+2+1+1=7~,
\end{align*}
which proves the desired Lipschitz bound.
The fact that for any $t\in[T]:h(\x_t)=0,\,\nabla h(\x_t)=\e_1$ is easily verified by construction and by \eqref{eq: nabla g}.
In order to finish the proof, we need to show that $h$ has no $(\delta,\frac{1}{36})$ stationary-points. 
By construction we have
\[
\partial h(\x)=
\begin{cases}
\sv\,,&\forall t\in[T]:\norm{\x-\x_t}>r
\\
\nabla g_{\x_t}(\x)\,,&\exists t\in[T]:\norm{\x-\x_t}< r
\end{cases}
~,
\]
while for $\norm{\x-\x_t}=r$ we would get convex combinations of the two cases.\footnote{Since we are interested in analyzing the $\delta$-subdifferential set which consists of convex combinations of subgradients, and subgradients are defined as convex combinations of gradients at differentiable points - it is enough to consider convex combinations of gradients at differentiable points in the first place.}
Inspecting the set $\{\nabla g_{\x_t}(\x):\norm{\x-\x_t}<r\}$ through \eqref{eq: nabla g}, we see that it depends on $\x,\x_t$ only through $\x-\x_t$ and that actually
\[\{\nabla g_{\x_t}(\x):\norm{\x-\x_t}<r\}=\{\nabla g_{\zero_d}(\x):\norm{\x}<r\}~,
\]
which is convenient since the latter set does not depend on $\x_t$. Overall, we see that
any gradient of $h$ is in the set
\begin{align*}
\mathcal{G}:=&\left\{
\lambda_1 \sv+ \lambda _2\left(\frac{2\sv^\top \x}{r^2}\x+\frac{\norm{\x}^2}{r^2}\sv
-\frac{2\e_1^\top \x}{r^2}\x
-\frac{\norm{\x}^2}{r^2}\e_1
+\e_1\right)
:\lambda_1,\lambda_2\geq0,\lambda_1+\lambda_2=1,\norm{\x}\leq r
\right\}
\\
=&\left\{
\lambda_1 \sv+ \lambda _2\left({2\sv^\top \x}\cdot \x+{\norm{\x}^2} \sv
-{2\e_1^\top \x}\cdot \x
-{\norm{\x}^2}\e_1
+\e_1\right)
:\lambda_1,\lambda_2\geq0,\lambda_1+\lambda_2=1,\norm{\x}\leq 1
\right\}
\\
=&\left\{
(\lambda_1+\lambda_2\norm{\x}^2) \sv
+2\lambda_2((\sv-\e_1)^\top \x) \x
+\lambda_2(1-\norm{\x}^2) \e_1
:\lambda_1,\lambda_2\geq0,\lambda_1+\lambda_2=1,\norm{\x}\leq 1
\right\}~.
\end{align*}
We aim to show that $\mathrm{conv}(\mathcal{G})$ does not contain any vectors of norm smaller than $\frac{1}{36}$. For $\su\in\mathcal{G}$ with corresponding $\lambda_1,\lambda_2,\x$ as above, it holds that
\begin{align} \label{eq: uv}
\su^\top \sv
&=\lambda_{1}+\lambda_{2}\norm{\x}^2+2\lambda_2(\sv-\e_1)^\top \x\cdot \x^\top \sv 
\nonumber\\
&=\lambda_{1}+\lambda_{2}\norm{\x}^2+2\lambda_2 (\sv^\top \x)^2-2\lambda_2\e_1^\top \x\cdot \x^\top \sv
\nonumber\\
&\geq \lambda_{1}+\lambda_{2}(\sv^\top \x)^2+\lambda_{2}(\e_1^\top \x)^2
+2\lambda_2 (\sv^\top \x)^2-2\lambda_2\e_1^\top \x\cdot \x^\top \sv
\nonumber\\
&=\lambda_{1}+\lambda_2(\sv^\top \x- \e_1^\top \x)^2
+2\lambda_2 (\sv^\top \x)^2
\nonumber\\
&\geq \lambda_{1}+\lambda_2(\sv^\top \x- \e_1^\top \x)^2
\nonumber\\
&\geq \lambda_2(\sv^\top \x- \e_1^\top \x)^2
~.
\end{align}
So for $\sxi\in\conv(\mathcal{G})$ represented as the convex combination $\sxi=\sum_{i=1}^{N}\mu^i \su^i$, with each $\su^i\in\mathcal{G}$ having its corresponding $\lambda_1^i,\lambda_2^i,\x^i$, we get

\begin{align} \label{eq: uv summed}
\sxi^\top \sv&=\sum_{i=1}^{N}\mu^i (\su^i)^\top \sv
\overset{(\ref{eq: uv})}{\geq}
\sum_{i=1}^{N}\mu^i \lambda_2^i((\sv-\e_1)^\top \x^i)^2
\overset{\text{Cauchy-Schwarz}}{\geq} \frac{\left(\sum_{i=1}^{N}\mu^i \lambda_2^i|(\sv-\e_1)^\top \x^i|\right)^2}{\sum_{i=1}^{N}\mu^i\lambda_2^i}
\nonumber
\\
&\geq
\left(\sum_{i=1}^{N}\mu^i \lambda_2^i|(\sv-\e_1)^\top \x^i|\right)^2
~,
\end{align}
where in the last inequality we used $\sum_{i=1}^N\mu^i\lambda_2^i\leq \max_{i}\lambda_2^i\cdot\sum_i\mu^i\leq1$. 
This further gives
\begin{align*}
\sxi^\top(\e_1+\sv)
&=\sum_{i=1}^{N}\mu^i (\e_1+\sv)^\top \su^i
\\&=\sum_{i=1}^{N}\mu^i [1+2\lambda^i_2(\sv^\top \x^i-\e_1^\top \x^i)(\e_1^\top \x^i+\sv^\top \x^i)]
\\
&=1+2\sum_{i=1}^{N}\mu^i \lambda^i_2((\sv-\e_1)^\top \x^i)((\sv+\e_1)^\top \x^i)
\\
&\overset{\|\x^i\|\leq 1}{\geq}
1-4\sum_{i=1}^{N}\mu^i \lambda^i_2|(\sv-\e_1)^\top \x^i|
\\
&\overset{(\ref{eq: uv summed})}{\geq} 1-4\sqrt{\sxi^\top \sv}
~.
\end{align*}
Hence,
given $\sxi\in\conv(\mathcal{G})$ we can assume that $\|\sxi\|\leq 1$ (since otherwise there is nothing left to show), and see that
\begin{gather*}
1
\leq |\sxi^\top (\e_1+\sv)|+4\sqrt{\lambda_2 \sxi^\top \sv}
\leq \norm{\sxi}\cdot \norm{\e_1+\sv}+4\sqrt{\norm{\sxi}}
\leq \sqrt{2}\norm{\sxi}+4\sqrt{\norm{\sxi}}
\overset{\norm{\sxi}<1}{\leq} \sqrt{2\norm{\sxi}}+4\sqrt{\norm{\sxi}}
\\
\implies
\norm{\sxi}\geq \frac{1}{(\sqrt{2}+4)^2}
>\frac{1}{36}~.
\end{gather*}

\end{proof}

Given the previous lemma we can easily finish the proof of the theorem by looking at
\[
f(\x):=\max\left\{\frac{L}{7}h(\x),-\Delta\right\}~.
\]
$f$ is $L$-Lipschitz (since $h$ is $7$-Lipschitz) and satisfies $f(\x_0)-\inf_{\x}f(\x)\leq 0-(-\Delta)=\Delta$, as required. Furthermore, for any $t\in[T]: h(\x_t)=0>-\Delta\implies f(\x_t)=\frac{L}{7}h(\x_t)=0$. Since $f$ is $L$-Lipschitz, this further implies that for any $\x$ such that $\norm{\x-\x_t}<\frac{\Delta}{L}:~f(\x)>-\Delta\implies \partial f(\x)=\frac{L}{7}\partial h(\x)$. In particular, 
$\partial_{\delta} f(\x_t)=\frac{L}{7}\partial_{\delta} h(\x_t)$
for any $\delta<\frac{\Delta}{L}$,
so the lemma shows that $\x_t$ is not a $(\delta,\epsilon)$ stationary point of $f$ for $\epsilon<\frac{L}{7}\cdot\frac{1}{36}=\frac{L}{252}$.

\subsection{Proof of \thmref{Theorem:1LB}} \label{inf_lower_proof}

Let $0<\delta<\epsilon<1$, and let $T<\infty$. It is enough to prove the case $d=1$, since otherwise we can simply look at $\x\mapsto f(x_1)$ with $f$ being the lower bound construction in one dimension.

Suppose that an algorithm has access only to a derivative oracle, and consider the case that for any $t\in[T]$ the oracles response is $f'(\x_t)=1$. Since the algorithm is deterministic this fixes the iterate sequence $Q:=(\x_1,\dots,\x_T)$.
Afterwards, the algorithm returns the candidate solution $\hat{\x}$ for being a $(\delta, \epsilon)$-Goldstein stationary point. We remark that $\hat{\x}$ might not be in $Q$.
We will show that the described resisting strategy is indeed consistent with a function which satisfies the conditions in the theorem. Namely, it suffices to construct a 1-Lipschitz function $f$ such that $ f'(\x_t) = 1$ for all $t\in[T]$ yet $\hat{\x}$ is not a $(\delta, \epsilon)$-Goldstein stationary point.

To that end, let $\eta\in(0,1-\delta)$ be such that $\hat{\x} + \delta + \eta \notin Q$ and $\hat{\x} - \delta - \eta \notin Q$ (recall that $Q$ is finite, thus such $\eta$ exists). We set $f(\x) = \x - \hat{\x}$ for all $\x \in [\hat{\x} - \delta + \eta, \hat{\x} + \delta + \eta]$, which ensures $\partial_\delta f(\hat{\x}) = \{1\}$, and in particular the norm of the minimal-norm element in $\partial_\delta f(\hat{\x})$ is $1$. Since $\epsilon < 1$, we get that $\hat{\x}$ is not a $(\delta, \epsilon)$-Goldstein stationary, as required.
Moreover, for all $\x_t \in Q \cap [\hat{\x} - \delta + \eta, \hat{\x} + \delta + \eta]$, we have $f'(\x_t) = 1$. Thus, for these query points that lie in the interval $[\hat{\x} - \delta + \eta, \hat{\x} + \delta + \eta]$, we satisfy the resisting oracle condition,
  
We continue on to define the function $f(\x)$ for any $\x >\hat{\x} + \delta + \eta$. The idea is to simply keep $f(\x) = \delta + \eta$ in this range while adding some small bumps to guarantee that $ f'(\x_t) = 1$ for all $\x_t \in Q \cap (\hat{\x} + \delta + \eta, \infty)$. Let $\bar{Q} = Q \cup \{\hat{\x} - \delta + \eta, \hat{\x} + \delta + \eta\}$ and $r_1 = \frac{1}{10} \min_{\x,\x' \in \bar{Q}, \x \neq \x'}\{|\x - \x'|\}$, we define $r = \min\{r_1, \delta\}$ and 
\begin{equation*}
f(\x) = 
\begin{cases}
\delta + \eta~, & \forall\x' \in Q:|\x - \x'| > r
\\
\delta + \eta - \x~, & \exists \x' \in Q :|\x - \x'| \leq r \textnormal{~~and~~} \x \leq \x' - \tfrac{r}{2}
\\
\delta + \eta - r + \x~, & \exists \x' \in Q : |\x - \x'| \leq r \textnormal{~~and~~} \x > \x' - \tfrac{r}{2} 
\end{cases}
~.
\end{equation*}
We see from the above definition that $0 \leq f(\x) \leq \delta + \eta$ for all $\x> \hat{\x} + \delta + \eta$ and $f'(\x) = 1$ for all $\x \in Q \cap (\hat{\x} + \delta + \eta, \infty)$. Similarly, we define $f(\x)$ for any $\x<- \hat{\x} - \delta - \eta$ as:
\begin{equation*}
f(\x) = 
\begin{cases}
- \delta - \eta~, & \forall \x' \in Q:|\x - \x'| > r 
\\
- \delta - \eta + \x~, & \exists \x' \in Q : |\x - \x'| \leq r \textnormal{~~and~~} \x \leq \x' + \tfrac{r}{2}
\\
- \delta - \eta + r - \x~, & \exists \x' \in Q :|\x - \x'| \leq r \textnormal{~~and~~} \x > \x' + \tfrac{r}{2}
\end{cases}
~. 
\end{equation*}
Putting all the pieces together, we get that $f$ is a $1$-Lipschitz function satisfying $f'(\x_t) = 1$ for all $t\in[T]$, and $\hat{\x}$ is not a $(\delta, \epsilon)$-Goldstein stationary point for any $0 < \delta < \epsilon < 1$, yielding the desired result.

\subsection{Proof of \thmref{Theorem:SMOOTH}} \label{alg_proof}

We start by concretely stating the purpose of the binary search given by Algorithm~\ref{Alg:BS}.

\begin{lemma}
Suppose $\x\in\reals^d$ and $\g_0\in\partial_\delta f(\x)$ are such that
$f(\x - \delta\tfrac{\g_0}{\|\g_0\|}) - f(\x) > - \tfrac{\delta}{2}\|\g_0\|$ and $\|\g_0\| > \epsilon$. Then $\textsc{Binary-Search}(\delta, \nabla f(\cdot), \g_0, \x)$ terminates within $O(\log(H\delta/\epsilon))$ first-order oracle calls and returns $\g_{\textnormal{new}}\in\partial_{\delta}f(\x)$ such that $\g_{\textnormal{new}}^\top \g_0\leq \frac{3}{4}\|\g_0\|^2$.
\end{lemma}

\begin{proof}
Using the first assumption on $\x,\,\g_0$ we apply the fundamental theorem of calculus to see that
\begin{align*}
\frac{1}{2}\norm{\g_0}^2
&\geq
\frac{\norm{\g_0}}{\delta}\left(f(\x)-f\left(\x - \delta\tfrac{\g_0}{\|\g_0\|}\right)\right)
=\frac{1}{\delta}\int_{0}^{\delta}\inner{\nabla f\left(\x-\frac{r}{\norm{\g_0}}\g_0\right),\g_0}dr
\\
&=\E_{\xi\sim U[\x,\x-\frac{\delta}{\norm{\g_0}}\g_0]}\left[\inner{\nabla f(\x+\xi),\g_0}\right]~,
\end{align*}
hence in expectation with respect to the uniform measure over the segment $[\x,\x-\frac{\delta}{\norm{\g_0}}\g_0]$, sampling a gradient will indeed satisfy the required condition. Applying the fundamental theorem again, it is easy to see that the \textbf{if} condition in Algorithm \ref{Alg:BS} checks whether the average gradient along the right half of the segment has larger inner product with $\g_0$ than other half or vise versa, and then continues examining the half with the smaller expected inner product. Thus, after $k$ iterations of this process we are left with a segment $I_k$ of length $2^{-k}\delta$ along which
\[
\E_{\xi\sim U[I_k]}\left[\inner{\nabla f(\x+\xi),\g_0}\right]\leq\frac{1}{2}\norm{\g_0}^2.
\]
But recalling that $\nabla f$ is $H$-Lipschitz,
we get that all gradients of $f$ over $I_k$ are at distance smaller than $H\cdot2^{-k}\delta$ from one another. In particular, for $k=O(\log(H\delta/\epsilon))$ we get that all $\xi\in I$ satisfy
\[
\inner{\nabla f(\x+\xi),\g_0}
\leq\frac{1}{2}\norm{\g_0}^2+\frac{\epsilon^2}{4}
\leq\frac{3}{4}\norm{\g_0}^2~,
\]
where we have applied the second assumption on $\g_0$. Thus the algorithm terminates, and returns $\g_{\textnormal{new}}$ satisfying the required condition.
\end{proof}

Having established the complexity and guarantee produced by the binary search subroutine, we are now ready to analyze $\textsc{Deterministic-Goldstein-SG}(\x_0,\delta,\epsilon)$. Since $\g_{\textnormal{new}}\in\partial_{\delta}f(\x_t)$ and $\partial_{\delta}f(\x)$ is a convex set, we observe that $\g(\x_t)\in\partial_{\delta}f(\x_t)$. Accordingly, we see that whenever the \textbf{while} loop terminates then either $\norm{\g(\x_t)}\leq\epsilon$, meaning that $\x_t$ is a $(\delta,\epsilon)$-Goldstein stationary point, or else $f(\x_{t+1})\leq f(\x_t)-\frac{\delta}{2}\norm{\g(\x_t)}<f(\x_t)-\frac{\delta\epsilon}{2}$. If the former occurs we are done, while the latter can occur at most $\frac{2\Delta}{\delta\epsilon}=O(\frac{\Delta}{\delta\epsilon})$ times by the assumption that $f(\x_0)-\inf_{\x}f(\x)\leq\Delta$.

Hence, it remains to show that the inner loop (lines \textsc{5-7}) is repeated at most $O\left(\frac{L^2\log(L/\epsilon)}{\epsilon^2}\right)$ times per outer loop (namely, per $t$) in order to obtain the desired complexity overall. To that end, assume that $f\left(\x_t - \delta\tfrac{\g(\x_t)}{\|\g(\x_t)\|}\right) - f(\x_t) > - \tfrac{\delta}{2}\|\g(\x_t)\|$
and $\|\g(\x_t)\| > \epsilon$. By the previous lemma, we know that $\g_{\textnormal{new}}^\top \g(\x_t)\leq \frac{3}{4}\|\g(\x_t)\|^2$.
But that being the case, we get by definition of $\h_t$ that for all $\lambda\in[0,1]:$
\begin{align*}
\norm{\h_t}^2
&\leq\norm{\g(\x_t)+\lambda(\g_{\textnormal{new}}-\g(\x_t))}^2
\\
&= \|\g(\x_t)\|^2 + 2\lambda\g(\x_t)^\top(\g_{\textnormal{new}} - \g(\x_t)) + \lambda^2\|\g_{\textnormal{new}} - \g(\x_t)\|^2
\\
&\leq \left(1-2\lambda\right)\|\g(\x_t)\|^2 +
2\lambda\g(\x_t)^\top\g_{\textnormal{new}}
+4L^2
\\
&\leq \left(1-\frac{\lambda}{2}\right)\|\g(\x_t)\|^2+4L^2~.
\end{align*}
By letting $\lambda=\frac{\norm{\g(\x)}^2}{16L^2}$ and recalling that $\epsilon\leq\norm{\g(\x)}\leq L$ we get
\[
\norm{\h_t}^2\leq \left(1-\frac{\epsilon^2}{64L^2}\right)\|\g(\x_t)\|^2~.
\]
Hence each iteration shrinks $\norm{\g(\x_t)}^2$ by a factors of $\left(1-\frac{\epsilon^2}{64L^2}\right)$. Since initially $\norm{\g(\x_t)}^2\leq L^2$, this can happen at most 
$O\left(\frac{L^2\log(L/\epsilon)}{\epsilon^2}\right)$ times before having $\norm{\g(\x_t)}^2<\epsilon^2$, as claimed.

\subsection{Proof of \thmref{thm:circuitSmoothing}} \label{app:circuitSmoothing}

We construct the function $g$ by using exactly the same neural arithmetic circuit of $f$, where we replace all the $\boldsymbol{\relu}$ gates with the $\boldsymbol{\softrelu}$ gates: 
\[
\softrelu_{a}(z)
=\begin{cases}
z~, & z\geq a
\\
\frac{(z+a)^2}{4a}~, & -a\leq z< a
\\
0~,& z< -a
\end{cases}
~,
\]
and note that
\[
\softrelu_{a}(z) 
= \mathbb{E}_{\xi \sim U[- a, a]}[\relu(z + \xi)]~.
\]
The following summarizes the properties of softrelu gates, and can be easily verified.

\begin{lemma} \label{lem:softmax}
We have that (i) $|\relu(z) - \softrelu_{a}(z)| \leq \frac{a}{4}$, (ii) $\softrelu_a(\cdot)$ is $1$-Lipschitz, and (iii) $\softrelu_a$ is $\frac{1}{2 a}$-smooth.
\end{lemma}

We prove the theorem inductively, going through the gates of $\calC$ one by one with respect to a topological sorting of $\calC$ which will remain fixed throughout the proof. Our goal is to compare the evaluation of the nodes, in the order of this topological sorting, in the circuit of $f$ and in the circuit of $g$ under Assumption~\ref{asp:recursiveLipschitz}. We denote by $f_i$ be the function evaluated in the node $i$ of the circuit of $f$, in the topological sorting of the circuit of $f$, and by $g_i$ the corresponding function evaluated in the node $i$ of the circuit of $g$. Let also $L_i > 0$ be the corresponding Lipschitz parameter of $f_i$ and $G_i > 0$ the corresponding value bound parameter. 

As to the base of our induction, we examine the input nodes. The input nodes and the constant gates for $f$ and $g$ both have the same value, and have gradient $\e_j$ for some $j$. Thus, they are $0$-smooth. Also, the input nodes are $1$-Lipschitz and are bounded in value by the diameter of $\calR$, while the constant nodes are clearly $0$-Lipschitz. This will serve as the basis of our induction.

Our inductive hypothesis is that the following is satisfied for $i:$ For any $\x \in \calR$, any $j < i$, it holds that (i) $|f_j(\x) - g_j(\x)| \leq \gamma_j$, (ii) $g_j$ is $S_j$-smooth, (iii) $g_j$ is $L_j$-Lipschitz, and (iv) $g_j$ is $G_j$-bounded. Then, we seek to prove that (i) $|f_i(\x) - g_i(\x)| \leq \gamma_i$ for all $\x \in \calR$, (ii) $g_i$ is $S_i$-smooth, (iii) $g_i$ is $L_i$-Lipschitz, and (iv) $g_i$ is $G_i$-bounded, while bounding $\gamma_i,S_i,L_i,G_i$ as functions of the previous parameters.

We consider different cases according to the type of node $i$: 
\begin{itemize}
\item \textbf{output node.} In this case, the value and Lipschitz constant of node $i$ is the same as of a node $j < i$. Thus, we have $|f_i(\x) - g_i(\x)| \le \gamma_j =: \gamma_i$ and obtain that $S_i = S_j$, $L_i = L_j$, and $G_i$ = $G_j$.
\item \textbf{$\times$ node.} In this case, there exists $j, k < i$ such that $f_i(\x) = f_j(\x) \cdot f_k(\x)$ and $g_i(\x) = g_j(\x) \cdot g_k(\x)$ which means that 
\begin{align*}
  |f_i(\x) - g_i(\x)| & = |f_j(\x) \cdot f_k(\x) - g_j(\x) \cdot g_k(\x)| \\
  & \le |f_j(\x)| \cdot |f_k(\x) - g_k(\x)| + |g_k(\x)| \cdot |f_j(\x) - g_j(\x)| \\
  & \le G_j \cdot \gamma_k + \gamma_j \cdot G_k =: \gamma_i~.
\end{align*}
Also, $S_i \le S_j \cdot G_k + G_j \cdot S_k + 2 L_j \cdot L_k$. The Lipschitz constant of $g_i$ is upper bounded by $L_j \cdot G_k + G_j \cdot L_k$ which is equal to $L_i$ by Assumption \ref{asp:recursiveLipschitz}. Thus, $g_i$ is $L_i$-Lipschitz. It is also easy to see that $g_i$ is $G_i = G_j \cdot G_k$ bounded. 

\item \textbf{$+$ node.} In this case, there exist $j, k < i$ such that $f_i(\x) = f_j(\x) + f_k(\x)$ and $g_i(\x) = g_j(\x) + g_k(\x)$ which means that $|f_i(\x) - g_i(\x)| \le \gamma_j + \gamma_k =: \gamma_i$. Also, $S_i \le S_j + S_k$. Then, the Lipschitz constant of $g_i$ is upper bounded by $L_j + L_k$ which is equal to $L_i$ by Assumption \ref{asp:recursiveLipschitz}. Thus, $g_i$ is $L_i$-Lipschitz and is also easy to see that it is $G_i = G_j + G_k$ bounded.

\item \textbf{$\relu$ node.} In this case, there exists $j < i$ such that $f_i(\x) = \relu(f_j(\x))$ and $g_i(\x) = \softrelu(g_j(\x))$. Using Lemma \ref{lem:softmax}, the triangle inequality and the fact that $\relu$ is $1$-Lipschitz, we have $|f_i(\x) - g_i(\x)| \le \frac{a}{4} + \gamma_j =:\gamma_i$. The next is to bound the smoothness $S_i$. By definition, we have
\begin{equation*}
\nabla g_i(\x) = \nabla \softrelu_a (g_j(\x)) = \softrelu_a'(g_j(\x)) \nabla g_j(\x)~,
\end{equation*}
hence
\begin{eqnarray*}
\norm{\nabla g_i(\x) - \nabla g_i(\y)}
& = & \|\softrelu_a'(g_j(\x)) \nabla g_j(\x) - \softrelu_a'(g_j(\y)) \nabla g_j(\y)\| \\
& \le & |\softrelu_a'(g_j(\x))|\cdot \|\nabla g_j(\x) - \nabla g_j(\y)\| 
\\
&&+~ |\softrelu_a'(g_j(\x)) - \softrelu_a'(g_j(\y))| \cdot\| \nabla g_j(\y)\| \\
& \leq & (G_j \cdot S_j + \frac{1}{2 a} L_j)\|\x - \y\|.  
\end{eqnarray*}
So, $S_i \le G_j \cdot S_j + \frac{1}{2 a} L_j$.
Also, due to the fact that $\softrelu$ is $1$-Lipschitz, the Lipschitzness of $g_i$ is upper bounded by $L_j$ which is equal to $L_i$ by Assumption \ref{asp:recursiveLipschitz}. Thus, $g_i$ is $L_i$-Lipschitz. Finally, $g_i$ is obviously $G_j$ bounded and hence $G_i$ bounded by Assumption \ref{asp:recursiveLipschitz}.
\end{itemize}
Note that the sequence of errors $\{\gamma_i\}_{i \geq 1}$ is increasing. Going through all the cases we considered above,  we see that
$\gamma_i \leq \frac{a}{4} + G_k \gamma_j + G_j \gamma_k \leq \frac{1}{a} + 2 G \gamma_{i - 1}$
which implies that $\gamma_i \le a \cdot (2 G)^i$.
Thus, we have
\[
|f(\x) - g(\x)| \leq a \cdot (2 G)^{s(\calC)}~.
\]
Similarly, we have that $S_i \leq 2 G \cdot S_{i - 1} + 2 L^2 + \frac{L}{a}$, so as long we set $\frac{1}{a}$ large enough compared to $2 L^2$ (which will indeed be the case later on) we can simplify this to  $S_i \leq 2 G \cdot S_{i - 1} + 2 \frac{L}{a}$, hence $S_i \le 2 \frac{L}{a} \cdot (2 G)^{i}$. Thus, we have 
\[
S_i \leq \frac{2 L}{a} \cdot (2 G)^{s(\calC)}~.
\]
\smallskip

Next, we proceed to prove the equivalence between the Goldstein stationary points of $f$ and $g$. Towards this goal we introduce some notation: let $n = s(\calC)$, let $\mathcal{B}_i$ be all the bias variables that are used in the topological ordering of the neural circuit in nodes before the node $i$ and $b$ be the total number of bias variables, i.e., $b = |\mathcal{B}_n|$. We relate any $(\delta_i,\epsilon_i)$-Goldstein stationary of $g_i$ to those of $f_i$ through the following lemma.

\begin{lemma} \label{lem:stationaryEquivalence}
  Let $\x \in \calR$. Then there exists two positive sequences $\delta_1, \dots, \delta_n>0$ and $\epsilon_1, \dots, \epsilon_n>0$, a sequence of vectors $\s_1, \dots, \s_n$, and a sequence of distributions $\mathcal{D}_1^{\x}, \dots, \mathcal{D}_n^{\x}$ supported on $[-a , a]^{s(\calC)}$ such that for all $i \in [s(\calC)]$ the following hold:
  \begin{itemize}
        \item $\delta_i \le (40 \cdot L \cdot G)^{s(\calC)} \cdot a,~~\epsilon_i \le (160 \cdot L \cdot G)^{3 s(\calC)} \cdot a$ and $\norm{\s_i}_2 \le \epsilon_i$.
      
      \item $\nabla g_i(\x) = \Exp_{\y \sim \mathcal{D}_i^{\x}}[\nabla f_i(\y)] + \s_i$.
      \item The support of $\mathcal{D}_i^{\x}$ has diameter $\delta_i$ and contains only points where $f_k$ is differentiable for all $k$. Moreover, the support of $\mathcal{D}_i^{\x}$ only contains points $\y$ for which either $y_i = x_i$ for all $i$ that correspond to input variables that are not biases, or are biases which are not used in the computation of $f_i$, $g_i$.
  \end{itemize}
\end{lemma}

\begin{proof}
  We prove this lemma by induction on $i$. For the base of the induction we observe that all the input and constant nodes are in the beginning of the topological ordering, and satisfy $g_i = f_i$, and that $f_i$ is differentiable. For this reason, input and constant nodes satisfy the lemma with $\mathcal{D}_i^{\x}$ equal to the Dirac delta distribution at $\x$, $\s_i = 0$, $\delta_i = 0$, $\epsilon_i = 0$. Now for the inductive step we assume that the lemma holds for all $j < i$ and we split into the following cases for the node $i$, depending on the type of the node in the circuit.
  \begin{itemize}
  \item \textbf{output node.} In this case, we have that $g_i = g_j$, $f_i = f_j$  hence the lemma follows immediately by the inductive hypothesis.
  
  \item \textbf{$+$ node.} In this case, we have that $g_i = g_j + g_k$, $f_i = f_j + f_k$. By the inductive hypothesis we get
  \[ \nabla g_j(\x) = \Exp_{\y \sim \mathcal{D}_j^{\x}}[\nabla f_j(\y)] + \s_j~, \quad \quad \quad \quad \nabla g_k(\x) = \Exp_{\y \sim \mathcal{D}_k^{\x}}[\nabla f_k(\y)] + \s_k~. \]
  Now let's assume without loss of generality that $j \ge k$, then we set $\mathcal{D}_i^{\x} = \mathcal{D}_j^{\x}$ and from the fact that $\mathcal{B}_j$ is a super set of $\mathcal{B}_k$ and from linearity of expectation we get
  \begin{align*}
    \nabla g_i(\x) &= \nabla g_j(\x) + \nabla g_k(\x) = \Exp_{\y \sim \mathcal{D}_i^{\x}}[\nabla f_j(\y) + \nabla f_k(\y)] + \s_j + \s_k 
    \\&= \Exp_{\y \sim \mathcal{D}_i^{\x}}[\nabla f_i(\y)] + \s_j + \s_k~,
  \end{align*}
   hence the lemma holds for this $i$ with $\delta_i \le \max\{\delta_j, \delta_k\},~\s_i = \s_j + \s_k,~\epsilon_i = \epsilon_j + \epsilon_k$.
  \item \textbf{$\times$ node.} In this case, we have $g_i = g_j \cdot g_k$, $f_i = f_j \cdot f_k$ and by inductive hypothesis we get that
  \[ \nabla g_j(\x) = \Exp_{\y \sim \mathcal{D}_j^{\x}}[\nabla f_j(\y)] + \s_j, \quad \quad \quad \quad \nabla g_k(\x) = \Exp_{\y \sim \mathcal{D}_k^{\x}}[\nabla f_k(\y)] + \s_k. \]
  Recall that $g_i$ is differentiable, hence
  \[
  \nabla g_i(\x) = g_j(\x) \nabla g_k(\x) + g_k(\x) \nabla g_j(\x)~.
  \]
  Also, because we will only consider points $\x$ for which all $f_i$'s are also differentiable we have that 
  \[
  \nabla f_i(\x) = f_j(\x) \nabla f_k(\x) + f_k(\x) \nabla f_j(\x)~.
  \]
  Let's assume without loss of generality that $j \ge k$, then we set $\mathcal{D}_i^{\x} = \mathcal{D}_j^{\x}$ and from the fact that $\mathcal{B}_j$ is a super set of $\mathcal{B}_k$ we have that
  \begin{align}
  \nabla g_j(\x) = \Exp_{\y \sim \mathcal{D}_i^{\x}}[\nabla f_j(\y)] + \s_j~, \quad \quad \quad \quad \nabla g_k(\x) = \Exp_{\y \sim \mathcal{D}_i^{\x}}[\nabla f_k(\y)] + \s_k~. \label{eq:intermediateFunction:1}
  \end{align}
  Using \eqref{eq:intermediateFunction:1}, the gradient of $g_i$ and linearity of expectation we have
  \begin{align}
  \nabla g_i(\x) & = g_j(\x) \Exp_{\y \sim \mathcal{D}_i^{\x}}[\nabla f_k(\y)] + g_k(\x) \Exp_{\y \sim \mathcal{D}_i^{\x}}[\nabla f_k(\y)] + g_k(\x) \s_j + g_j(\x) \s_k \nonumber
  \\ 
  & = \Exp_{\y \sim \mathcal{D}_i^{\x}}[g_j(\x) \nabla f_k(\y) + g_k(\x) \nabla f_k(\y)] + g_k(\x) \s_j + g_j(\x) \s_k~.
\end{align}
At this point we invoke Assumption \ref{asp:recursiveLipschitz} to utilize that $g_j$, $g_k$ are $G$-bounded and that $f_j$, $f_k$ are $L$-Lipschitz and hence we have that
  \begin{align}
  \norm{\nabla g_i(\x) - \Exp_{\y \sim \mathcal{D}_i^{\x}}[\nabla f_i(\y)]}_2 & \le L \cdot \left( \Exp_{\y \sim \mathcal{D}_i^{\x}}[|g_j(\x) - f_j(\y)|] + \Exp_{\y \sim \mathcal{D}_i^{\x}}[|g_k(\x) - f_k(\y)|] \right) \nonumber \\
  & \quad \quad + G \cdot (\epsilon_j + \epsilon_k)~.
  \label{eq:intermediateFunction:2}
  \end{align}
  So it remains to bound $\Exp_{\y \sim \mathcal{D}_i^{\x}}[|g_j(\x) - f_j(\y)|]$ and $\Exp_{\y \sim \mathcal{D}_i^{\x}}[|g_k(\x) - f_k(\y)|]$. We will prove an upper bound on $\Exp_{\y \sim \mathcal{D}_i^{\x}}[|g_j(\x) - f_j(\y)|]$, while the same upper bound will work for for $k$ as well. First, observe that because of the structure of the biases and the definition of $\mathcal{D}_i^{\x}$ we have that
  \[\Exp_{\y \sim \mathcal{D}_i^{\x}}[|g_j(\x) - f_j(\y)|] = \Exp_{\y \sim \mathcal{D}_j^{\x}}[|g_j(\x) - f_j(\y)|]~.\]
  Now we get that
  \begin{align}
    \Exp_{\y \sim \mathcal{D}_j^{\x}}[|g_j(\x) - f_j(\y)|] & \le \Exp_{\y \sim \mathcal{D}_j^{\x}}[|g_j(\x) - g_j(\y)|] + \Exp_{\y \sim \mathcal{D}_j^{\x}}[|g_j(\y) - f_j(\y)|]~.
    \end{align}
    We can now use the first statement of the theorem that we have already proved to recall that $|g_j(\y) - f_j(\y)| \le \gamma_j \le (2 G)^j \cdot a$, and we can also use the Lipschitz constant of $g_j$ to get that
    \begin{align}
    \Exp_{\y \sim \mathcal{D}_j^{\x}}[|g_j(\x) - f_j(\y)|] & \le L \cdot \delta_j + (2 G)^i \cdot a ~.
    \label{eq:intermediateFunction:3}
  \end{align}
  Combining \eqref{eq:intermediateFunction:2} and \eqref{eq:intermediateFunction:3} we obtain
  \[ \norm{\nabla g_i(\x) - \Exp_{\y \sim \mathcal{D}_i^{\x}}[\nabla f_i(\y)]}_2 \le L^2 \cdot \delta_j + L \cdot (2 G)^i \cdot a + G \cdot (\epsilon_j + \epsilon_k)~, \]
  and the lemma follows for this case as well with $\delta_i \le \max\{\delta_j, \delta_k\}$ and $\epsilon_i \le  L^2 \cdot \delta_j + L \cdot (2 G)^i \cdot a + G \cdot (\epsilon_j + \epsilon_k)$.
  \item \textbf{$\relu$ node.} In this case we use the structure of the bias variables, and see that $g_i(\x) = \softrelu(g_j(\x) + x_{b_i})$ and $f_i(\x) = \relu(f_j(\x) + x_{b_i})$ where $b_i$ is the index of the vector $\x$ that corresponds to the bias variable appearing only in the relu-node $i$. From the definition of softrelu we have that $g_i(\x) = \Exp_{u \sim U[-a, a]}[\relu(g_j(\x) + x_{b_i} + u)]$. Using the fact that we only focus on $\x$ for which $f_j$ is differentiable we get that
  \begin{align*}
  \nabla g_i(\x) & = \Exp_{u \sim U[-a, a]}[\mathbf{1}\{g_j(\x) + x_{b_i} + u \ge 0\} (\nabla (g_j(\x) + x_{b_i}))] \\
  & = \Exp_{u \sim U[-a, a]}[\mathbf{1}\{g_j(\x) + x_{b_i} + u \ge 0\}] (\nabla (g_j(\x) + x_{b_i})) \\
  & = \Pr_{u \sim U[-a, a]}(g_j(\x) + x_{b_i} + u \ge 0) (\nabla (g_j(\x) + x_{b_i}))~, \\
  \end{align*}
  and also
  \[ \nabla f_i(\x) = \mathbf{1}\{f_j(\x) + x_{b_i} \ge 0\} (\nabla (f_j(\x) + x_{b_i}))~. \]
  From the inductive hypothesis and the fact that every bias variable appears only once we get
  \[ \nabla (g_j(\x) + x_{b_i}) = \Exp_{\y \sim \mathcal{D}_i^{\x}}[\nabla (f_j(\y) + x_{b_i})] + \s_j~, \]
  so by denoting $\zeta_i(\x) = \Pr_{u \sim U[-a, a]}(g_j(\x) + x_{b_i} + u \ge 0)$ the above implies that
  \begin{align*}
  \nabla g_i(\x) & = \zeta_i(\x) \cdot (\Exp_{\y \sim \mathcal{D}_i^{\x}}[\nabla (f_j(\y) + x_{b_i})] + \s_j)~.
  \end{align*}
  We now need to distinguish several cases: (1) $g_j(\x) + x_{b_i} \ge a + \gamma_j + L \delta_j$, (2) $g_j(\x) + x_{b_i} \le - a - \gamma_j - L \delta_j$, and (3) $|g_j(\x) + x_{b_i}| \le a + \gamma_j + L \delta_j$. 
  
  We start with the first case. If $g_j(\x) + x_{b_i} \ge a + \gamma_j + L \delta_j$ then this means that $\zeta_i(\x) = 1$ and that $f_j(\y) + x_{b_i} \ge 0$ for all $\y$ that are $\delta_j$-close to $\x$. This implies that we can choose $\mathcal{D}_i^{\x} = \mathcal{D}_j^{\x}$ and we immediately get
  \[ \nabla g_i(\x) = \Exp_{\y \sim \mathcal{D}_i^{\x}}[\nabla f_i(\y)] + \s_j ~,\]
  hence the lemma holds with $\delta_i = \delta_j$, $\epsilon_i = \epsilon_j$ and $\s_i = \s_j$.

  Similarly, for the second case we have that $\zeta_i(\x) = 0$ and $f_j(\y) + x_{b_i} \le 0$ for all $\y$ that are $\delta_j$-close to $\x$. In this case we have that
  \[ \nabla g_i(\x) = 0 = \Exp_{\y \sim \mathcal{D}_i{\x}}[\nabla f_i(\y)]~, \]
  hence the lemma holds with $\delta_i = \delta_j$, $\epsilon_i = \epsilon_j$ and $\s_i = 0$.

  Finally, we consider the case $|g_j(\x) + x_{b_i}| \le a + \gamma_j + L \delta_j$. This implies that $|f_j(\y) + x_{b_i}| \le a + 2 \gamma_j + 2 L \delta_j$ for all $\y$ that are $\delta_j$-close to $\x$. We define the distribution $\mathcal{D}_i^{\x}$ as follows: we first sample $\y$ from $\mathcal{D}_j^{\x}$, and let $\y_{-b_j}$ be the vector $\y$ with all the coordinates but $b_j$. We then sample $y'_{b_j}$ from a distribution such that with probability $\zeta_i(\x)$ it holds that $f(\y) + y'_{b_j} \ge \eta$ for some value $\eta > 0$ and with probability $1 - \zeta_i(\x)$ it holds that $f(\y) + y'_{b_j} \le - \eta$. We then observe the following: (a) by Rademacher's theorem we now that all $f_j$'s are almost everywhere differentiable, so for an arbitrarily small value $\eta$ and for every $\y$ we can find values $y'_{b_j}$ such that what we want holds and also all the functions $f_j$ are differentiable in $(\y_{-b_j}, y'_{b_j})$, and (b) the desired values $y'_{b_j}$ are at most $a + 2 \gamma_j + 2 L \delta_j$ away from $x_{b_j}$, hence at most $2 a + 4 \gamma_j + 4 L \delta_j$ away from each other. Also, from the definition of $\mathcal{D}_i^{\x}$ we have that
  \begin{align*}
    \Exp_{\y \sim \mathcal{D}_i^{\x}}[\nabla f_i(\y)] & = \Exp_{\y \sim \mathcal{D}_i^{\x}}[\mathbf{1}\{f_j(\y) + y_{b_i} \ge 0\} (\nabla f_j(\y) + \e_{b_i})] \\
    & = \Exp_{\y \sim \mathcal{D}_j^{\x}}[\Exp_{y'_{b_i}}[\mathbf{1}\{f_j(\y_{-b_i}) + y'_{b_i} \ge 0\} (\nabla f_j(\y) + \e_{b_i})]] \\
    & = \Exp_{\y \sim \mathcal{D}_j^{\x}}[\Exp_{y'_{b_i}}[\mathbf{1}\{f_j(\y_{-b_i}) + y'_{b_i} \ge 0\}] (\nabla f_j(\y) + \e_{b_i})] \\
    & = \Exp_{\y \sim \mathcal{D}_j^{\x}}[\Pr_{y'_{b_i}}(f_j(\y_{-b_i}) + y'_{b_i} \ge 0) (\nabla f_j(\y) + \e_{b_i})] \\
    & = \Exp_{\y \sim \mathcal{D}_j^{\x}}[\zeta_i(\x) (\nabla f_j(\y) + \e_{b_i})] \\
    & = \zeta_i(\x) \Exp_{\y \sim \mathcal{D}_j^{\x}}[(\nabla f_j(\y) + \e_{b_i})] \\
    & = \zeta_i(\x) (\nabla g_j(\y) + \e_{b_i}) - \zeta_i(\x) \s_j \\
    & = \nabla g_i(\x) - \zeta_i(\x) \s_j~,
  \end{align*}
  where in the first line we used the definition of the distribution $\mathcal{D}_i^{\x}$, in the second line we use the fact that the bias variable $y_{b_i}$ does not appear in the computation of $f_j$, in the fourth line we use the definition of the distribution of $y'_{b_i}$ given $\y$, and in the rest we use the definition of $\zeta_i$ and our inductive hypothesis. Overall we get that the lemma follows with $\delta_i \le \delta_j + 2 a + 4 \gamma_j + 4 L \delta_j$, $\epsilon_i \le \epsilon_j$, and $\s_i = \zeta_i(\x) \cdot \s_j$.
\end{itemize}
  To conclude, using the fact that both $\delta_i$ and $\epsilon_i$ are increasing according to the definitions above and using the worst bounds from all these cases we get 
  \[ \delta_i \le \delta_{i - 1} (4 L + 1) + (8 G)^{s(\calC)} a~, \]
  implying that 
  \[\delta_n \le (40 \cdot L \cdot G)^{s(\calC)} \cdot a~. \]
  Using this bound we can compute the worst possible bound for $\epsilon_i$, as we have
  \begin{align*}
  &\epsilon_i \le  L^2 \cdot \delta_j + L \cdot (2 G)^i \cdot a + 2 G \cdot \epsilon_{i - 1}
  \\
  \implies
  &\epsilon_i \le (80 \cdot L \cdot G)^{2 s(\calC)} \cdot a + 2 G \cdot \epsilon_{i - 1}
  \end{align*}
  implying that
  \[ \epsilon_n \le (160 \cdot L \cdot G)^{3 s(\calC)} \cdot a ~,\]
  hence the lemma follows.
\end{proof}

Overall, for all the analyzed quantities we get
\begin{align*}
  \gamma_n & \le a \cdot (2 G)^{s(\calC)}~, \\
  \delta_n & \le a\cdot(40 L  G)^{s(\calC)} ~,\\
  \epsilon_n & \le a\cdot(160  L  G)^{3 s(\calC)}~, \\
  S_n & \le \frac{2 L}{a} \cdot (2 G)^{s(\calC)}
  ~.
\end{align*}
Setting $a = (160 \cdot L \cdot G)^{3 s(\calC)} \cdot \gamma$ finishes the proof.

\section{Conclusion}\label{sec:conclu}

We have provided lower and upper bounds on the complexity of finding an approximate $(\delta, \epsilon)$-Goldstein stationary point of a Lipschitz function in deterministic nonsmooth and nonconvex optimization.
We have shown that unlike dimension-free randomized algorithms, any deterministic first-order algorithm must suffer from a nontrivial dimension dependence, by establishing a lower bound of $\Omega(d)$ for any dimension $d$, whenever $\delta, \epsilon > 0$ are smaller than given constants. Furthermore, we established the importance of a zeroth-order oracle in deterministic nonsmooth nonconvex optimization,
by proving that any deterministic algorithm that uses only a gradient oracle cannot guarantee to return an adequate point within any finite time.
Both lower bounds stand in contrast to randomized algorithms, as well as deterministic smooth nonconvex and nonsmooth convex settings, emphasizing the unique difficulty of nonsmooth nonconvex optimization.

We have also provided a deterministic algorithm that achieves the best known dimension-free rate with merely a logarithmic smoothness dependence, allowing de-randomization for slightly-smooth functions. 
This motivated the study of deterministic smoothings, in order to apply our algorithm for nonsmooth problems.
We proved that unlike existing randomized smoothings, no efficient deterministic black-box smoothing can provide any meaningful guarantees, providing an answer to an open question raised in the literature. Moreover, we have bypassed this impossibility result in a practical white-box model, providing a deterministic smoothing for a wide variety of widely used neural network architectures which is provably meaningful from an optimization viewpoint. Combined with our algorithm, this yields the first deterministic, dimension-free algorithm for optimizing such networks, circumventing our lower bound.

As to future directions, it is interesting to note that our lower bound for deterministic first-order optimization is linear with respect to the dimension, though we are not aware of any such algorithm with sub-exponential dimension dependence (namely, better than exhaustive grid-search). Therefore, we pose the following question:
\begin{quote}
\textbf{Open problem:} \textit{Is there a deterministic first-order algorithm for nonsmooth nonconvex optimization that returns a $(\delta,\epsilon)$-Goldstein stationary point using $\mathrm{poly}(d, \delta^{-1}, \epsilon^{-1}$) oracle calls?}
\end{quote}

\section*{Acknowledgments}
MJ and TL were supported in part by the Mathematical Data Science program of the Office of Naval Research under grant number N00014-18-1-2764 and by the Vannevar Bush Faculty Fellowship program
under grant number N00014-21-1-2941. MZ was supported by the Army Research Office (ARO) under contract W911NF-17-1-0304 as
part of the collaboration between US DOD, UK MOD and UK Engineering and Physical
Research Council (EPSRC) under the Multidisciplinary University Research Initiative (MURI).
GK and OS were supported in part by the European Research Council (ERC) grant 754705, and by an Israeli Council for Higher Education grant via the Weizmann Data Science Research Center.

\bibliographystyle{plainnat}
\bibliography{ref}

\begin{thebibliography}{58}
\providecommand{\natexlab}[1]{#1}
\providecommand{\url}[1]{\texttt{#1}}
\expandafter\ifx\csname urlstyle\endcsname\relax
  \providecommand{\doi}[1]{doi: #1}\else
  \providecommand{\doi}{doi: \begingroup \urlstyle{rm}\Url}\fi

\bibitem[Arjevani et~al.(2020)Arjevani, Carmon, Duchi, Foster, Sekhari, and
  Sridharan]{Arjevani-2020-Second}
Y.~Arjevani, Y.~Carmon, J.~C. Duchi, D.~J. Foster, A.~Sekhari, and
  K.~Sridharan.
\newblock Second-order information in non-convex stochastic optimization: Power
  and limitations.
\newblock In \emph{COLT}, pages 242--299. PMLR, 2020.

\bibitem[Arjevani et~al.(2022)Arjevani, Carmon, Duchi, Foster, Srebro, and
  Woodworth]{Arjevani-2022-Lower}
Y.~Arjevani, Y.~Carmon, J.~C. Duchi, D.~J. Foster, N.~Srebro, and B.~Woodworth.
\newblock Lower bounds for non-convex stochastic optimization.
\newblock \emph{Mathematical Programming}, pages 1--50, 2022.

\bibitem[Attouch and Aze(1993)]{Attouch-1993-Approximation}
H.~Attouch and D.~Aze.
\newblock Approximation and regularization of arbitrary functions in {H}ilbert
  spaces by the {L}asry-{L}ions method.
\newblock \emph{Annales de l'Institut Henri Poincar{\'e} C, Analyse non
  lin{\'e}aire}, 10\penalty0 (3):\penalty0 289--312, 1993.

\bibitem[Beck and Teboulle(2012)]{Beck-2012-Smoothing}
A.~Beck and M.~Teboulle.
\newblock Smoothing and first order methods: A unified framework.
\newblock \emph{SIAM Journal on Optimization}, 22\penalty0 (2):\penalty0
  557--580, 2012.

\bibitem[Bena{\"\i}m et~al.(2005)Bena{\"\i}m, Hofbauer, and
  Sorin]{Benaim-2005-Stochastic}
M.~Bena{\"\i}m, J.~Hofbauer, and S.~Sorin.
\newblock Stochastic approximations and differential inclusions.
\newblock \emph{SIAM Journal on Control and Optimization}, 44\penalty0
  (1):\penalty0 328--348, 2005.

\bibitem[Bolte and Pauwels(2021)]{Bolte-2021-Conservative}
J.~Bolte and E.~Pauwels.
\newblock Conservative set valued fields, automatic differentiation, stochastic
  gradient methods and deep learning.
\newblock \emph{Mathematical Programming}, 188\penalty0 (1):\penalty0 19--51,
  2021.

\bibitem[Braun et~al.(2017)Braun, Guzm{\'a}n, and Pokutta]{Braun-2017-Lower}
G.~Braun, C.~Guzm{\'a}n, and S.~Pokutta.
\newblock Lower bounds on the oracle complexity of nonsmooth convex
  optimization via information theory.
\newblock \emph{IEEE Transactions on Information Theory}, 63\penalty0
  (7):\penalty0 4709--4724, 2017.

\bibitem[Burke et~al.(2002{\natexlab{a}})Burke, Lewis, and
  Overton]{Burke-2002-Approximating}
J.~V. Burke, A.~S. Lewis, and M.~L. Overton.
\newblock Approximating subdifferentials by random sampling of gradients.
\newblock \emph{Mathematics of Operations Research}, 27\penalty0 (3):\penalty0
  567--584, 2002{\natexlab{a}}.

\bibitem[Burke et~al.(2002{\natexlab{b}})Burke, Lewis, and
  Overton]{Burke-2002-Two}
J.~V. Burke, A.~S. Lewis, and M.~L. Overton.
\newblock Two numerical methods for optimizing matrix stability.
\newblock \emph{Linear Algebra and its Applications}, 351:\penalty0 117--145,
  2002{\natexlab{b}}.

\bibitem[Burke et~al.(2005)Burke, Lewis, and Overton]{Burke-2005-Robust}
J.~V. Burke, A.~S. Lewis, and M.~L. Overton.
\newblock A robust gradient sampling algorithm for nonsmooth, nonconvex
  optimization.
\newblock \emph{SIAM Journal on Optimization}, 15\penalty0 (3):\penalty0
  751--779, 2005.

\bibitem[Burke et~al.(2020)Burke, Curtis, Lewis, Overton, and
  Sim{\~o}es]{Burke-2020-Gradient}
J.~V. Burke, F.~E. Curtis, A.~S. Lewis, M.~L. Overton, and L.~E.~A. Sim{\~o}es.
\newblock Gradient sampling methods for nonsmooth optimization.
\newblock \emph{Numerical Nonsmooth Optimization: State of the Art Algorithms},
  pages 201--225, 2020.

\bibitem[Carmon et~al.(2020)Carmon, Duchi, Hinder, and
  Sidford]{Carmon-2020-Lower}
Y.~Carmon, J.~C. Duchi, O.~Hinder, and A.~Sidford.
\newblock Lower bounds for finding stationary points {I}.
\newblock \emph{Mathematical Programming}, 184\penalty0 (1):\penalty0 71--120,
  2020.

\bibitem[Carmon et~al.(2021)Carmon, Duchi, Hinder, and
  Sidford]{Carmon-2021-Lower}
Y.~Carmon, J.~C. Duchi, O.~Hinder, and A.~Sidford.
\newblock Lower bounds for finding stationary points {II}: first-order methods.
\newblock \emph{Mathematical Programming}, 185\penalty0 (1):\penalty0 315--355,
  2021.

\bibitem[Cartis et~al.(2010)Cartis, Gould, and Toint]{Cartis-2010-Complexity}
C.~Cartis, N.~I.~M. Gould, and P.~L. Toint.
\newblock On the complexity of steepest descent, {N}ewton's and regularized
  newton's methods for nonconvex unconstrained optimization problems.
\newblock \emph{SIAM Journal on Optimization}, 20\penalty0 (6):\penalty0
  2833--2852, 2010.

\bibitem[Cartis et~al.(2012)Cartis, Gould, and Toint]{Cartis-2012-Complexity}
C.~Cartis, N.~I.~M. Gould, and P.~L. Toint.
\newblock Complexity bounds for second-order optimality in unconstrained
  optimization.
\newblock \emph{Journal of Complexity}, 28\penalty0 (1):\penalty0 93--108,
  2012.

\bibitem[Cartis et~al.(2018)Cartis, Gould, and Toint]{Cartis-2018-Worst}
C.~Cartis, N.~I.~M. Gould, and P.~L. Toint.
\newblock Worst-case evaluation complexity and optimality of second-order
  methods for nonconvex smooth optimization.
\newblock In \emph{Proceedings of the International Congress of Mathematicians:
  Rio de Janeiro}, pages 3711--3750. World Scientific, 2018.

\bibitem[Chen(2012)]{Chen-2012-Smoothing}
X.~Chen.
\newblock Smoothing methods for nonsmooth, nonconvex minimization.
\newblock \emph{Mathematical Programming}, 134\penalty0 (1):\penalty0 71--99,
  2012.

\bibitem[Clarke(1974)]{Clarke-1974-Necessary}
F.~H. Clarke.
\newblock Necessary conditions for nonsmooth variational problems.
\newblock In \emph{Optimal Control Theory and Its Applications}, pages 70--91.
  Springer, 1974.

\bibitem[Clarke(1975)]{Clarke-1975-Generalized}
F.~H. Clarke.
\newblock Generalized gradients and applications.
\newblock \emph{Transactions of the American Mathematical Society},
  205:\penalty0 247--262, 1975.

\bibitem[Clarke(1981)]{Clarke-1981-Generalized}
F.~H. Clarke.
\newblock Generalized gradients of {L}ipschitz functionals.
\newblock \emph{Advances in Mathematics}, 40\penalty0 (1):\penalty0 52--67,
  1981.

\bibitem[Clarke(1990)]{Clarke-1990-Optimization}
F.~H. Clarke.
\newblock \emph{Optimization and Nonsmooth Analysis}.
\newblock SIAM, 1990.

\bibitem[Clarke et~al.(2008)Clarke, Ledyaev, Stern, and
  Wolenski]{Clarke-2008-Nonsmooth}
F.~H. Clarke, Y.~S. Ledyaev, R.~J. Stern, and P.~R. Wolenski.
\newblock \emph{Nonsmooth Analysis and Control Theory}, volume 178.
\newblock Springer Science \& Business Media, 2008.

\bibitem[Daniilidis and Drusvyatskiy(2020)]{Daniilidis-2020-Pathological}
A.~Daniilidis and D.~Drusvyatskiy.
\newblock Pathological subgradient dynamics.
\newblock \emph{SIAM Journal on Optimization}, 30\penalty0 (2):\penalty0
  1327--1338, 2020.

\bibitem[Daskalakis and Papadimitriou(2011)]{Daskalakis-2011-Continuous}
C.~Daskalakis and C.~Papadimitriou.
\newblock Continuous local search.
\newblock In \emph{SODA}, pages 790--804. SIAM, 2011.

\bibitem[Davis and Drusvyatskiy(2019)]{davis2019stochastic}
D.~Davis and D.~Drusvyatskiy.
\newblock Stochastic model-based minimization of weakly convex functions.
\newblock \emph{SIAM Journal on Optimization}, 29\penalty0 (1):\penalty0
  207--239, 2019.

\bibitem[Davis et~al.(2020)Davis, Drusvyatskiy, Kakade, and
  Lee]{Davis-2020-Stochastic}
D.~Davis, D.~Drusvyatskiy, S.~Kakade, and J.~D. Lee.
\newblock Stochastic subgradient method converges on tame functions.
\newblock \emph{Foundations of Computational Mathematics}, 20\penalty0
  (1):\penalty0 119--154, 2020.

\bibitem[Davis et~al.(2022)Davis, Drusvyatskiy, Lee, Padmanabhan, and
  Ye]{Davis-2022-Gradient}
D.~Davis, D.~Drusvyatskiy, Y.~T. Lee, S.~Padmanabhan, and G.~Ye.
\newblock A gradient sampling method with complexity guarantees for {L}ipschitz
  functions in high and low dimensions.
\newblock In \emph{NeurIPS}, 2022.

\bibitem[Duchi et~al.(2012)Duchi, Bartlett, and
  Wainwright]{Duchi-2012-Randomized}
J.~C. Duchi, P.~L. Bartlett, and M.~J. Wainwright.
\newblock Randomized smoothing for stochastic optimization.
\newblock \emph{SIAM Journal on Optimization}, 22\penalty0 (2):\penalty0
  674--701, 2012.

\bibitem[Fearnley et~al.(2021)Fearnley, Goldberg, Hollender, and
  Savani]{Fearnley-2021-Complexity}
J.~Fearnley, P.~W. Goldberg, A.~Hollender, and R.~Savani.
\newblock The complexity of gradient descent: {CLS} = {PPAD} $\cap$ {PLS}.
\newblock In \emph{STOC}, pages 46--59, 2021.

\bibitem[Ghadimi and Lan(2013)]{ghadimi2013stochastic}
S.~Ghadimi and G.~Lan.
\newblock Stochastic first-and zeroth-order methods for nonconvex stochastic
  programming.
\newblock \emph{SIAM Journal on Optimization}, 23\penalty0 (4):\penalty0
  2341--2368, 2013.

\bibitem[Glorot et~al.(2011)Glorot, Bordes, and Bengio]{Glorot-2011-Deep}
X.~Glorot, A.~Bordes, and Y.~Bengio.
\newblock Deep sparse rectifier neural networks.
\newblock In \emph{AISTATS}, pages 315--323, 2011.

\bibitem[Goldstein(1977)]{Goldstein-1977-Optimization}
A.~Goldstein.
\newblock Optimization of {L}ipschitz continuous functions.
\newblock \emph{Mathematical Programming}, 13\penalty0 (1):\penalty0 14--22,
  1977.

\bibitem[Goodfellow et~al.(2016)Goodfellow, Bengio, and
  Courville]{Goodfellow-2016-Deep}
I.~Goodfellow, Y.~Bengio, and A.~Courville.
\newblock \emph{Deep Learning}.
\newblock MIT Press, 2016.

\bibitem[Guzm{\'a}n and Nemirovski(2015)]{Guzman-2015-Lower}
C.~Guzm{\'a}n and A.~Nemirovski.
\newblock On lower complexity bounds for large-scale smooth convex
  optimization.
\newblock \emph{Journal of Complexity}, 31\penalty0 (1):\penalty0 1--14, 2015.

\bibitem[Ioffe and Szegedy(2015)]{Ioffe-2015-Batch}
S.~Ioffe and C.~Szegedy.
\newblock Batch normalization: Accelerating deep network training by reducing
  internal covariate shift.
\newblock In \emph{ICML}, pages 448--456. PMLR, 2015.

\bibitem[Kiwiel(1996)]{Kiwiel-1996-Restricted}
K.~C. Kiwiel.
\newblock Restricted step and {L}evenberg-{M}arquardt techniques in proximal
  bundle methods for nonconvex nondifferentiable optimization.
\newblock \emph{SIAM Journal on Optimization}, 6\penalty0 (1):\penalty0
  227--249, 1996.

\bibitem[Kiwiel(2007)]{Kiwiel-2007-Convergence}
K.~C. Kiwiel.
\newblock Convergence of the gradient sampling algorithm for nonsmooth
  nonconvex optimization.
\newblock \emph{SIAM Journal on Optimization}, 18\penalty0 (2):\penalty0
  379--388, 2007.

\bibitem[Kornowski and Shamir(2021)]{Kornowski-2021-Oracle}
G.~Kornowski and O.~Shamir.
\newblock Oracle complexity in nonsmooth nonconvex optimization.
\newblock In \emph{NeurIPS}, pages 324--334, 2021.

\bibitem[Kornowski and Shamir(2022)]{kornowski2022oracle}
G.~Kornowski and O.~Shamir.
\newblock Oracle complexity in nonsmooth nonconvex optimization.
\newblock \emph{Journal of Machine Learning Research}, 23\penalty0
  (314):\penalty0 1--44, 2022.

\bibitem[Lasry and Lions(1986)]{Lasry-1986-Remark}
J-M. Lasry and P-L. Lions.
\newblock A remark on regularization in {H}ilbert spaces.
\newblock \emph{Israel Journal of Mathematics}, 55\penalty0 (3):\penalty0
  257--266, 1986.

\bibitem[LeCun et~al.(2015)LeCun, Bengio, and Hinton]{Lecun-2015-Deep}
Y.~LeCun, Y.~Bengio, and G.~Hinton.
\newblock Deep learning.
\newblock \emph{Nature}, 521\penalty0 (7553):\penalty0 436--444, 2015.

\bibitem[Lin et~al.(2022)Lin, Zheng, and Jordan]{Lin-2022-Gradient}
T.~Lin, Z.~Zheng, and M.~I. Jordan.
\newblock Gradient-free methods for deterministic and stochastic nonsmooth
  nonconvex optimization.
\newblock In \emph{NeurIPS}, 2022.

\bibitem[M{\"a}kel{\"a} and Neittaanm{\"a}ki(1992)]{Makela-1992-Nonsmooth}
M.~M. M{\"a}kel{\"a} and P.~Neittaanm{\"a}ki.
\newblock \emph{Nonsmooth Optimization: Analysis and Algorithms with
  Applications to Optimal Control}.
\newblock World Scientific, 1992.

\bibitem[Miyato et~al.(2018)Miyato, Kataoka, Koyama, and
  Yoshida]{Miyato-2018-Spectral}
T.~Miyato, T.~Kataoka, M.~Koyama, and Y.~Yoshida.
\newblock Spectral normalization for generative adversarial networks.
\newblock In \emph{ICLR}, 2018.

\bibitem[Murty and Kabadi(1987)]{Murty-1987-Some}
K.~G. Murty and S.~N. Kabadi.
\newblock Some {NP}-complete problems in quadratic and nonlinear programming.
\newblock \emph{Mathematical Programming}, 39\penalty0 (2):\penalty0 117--129,
  1987.

\bibitem[Nair and Hinton(2010)]{Nair-2010-Rectified}
V.~Nair and G.~E. Hinton.
\newblock Rectified linear units improve restricted {B}oltzmann machines.
\newblock In \emph{ICML}, pages 807--814, 2010.

\bibitem[Nemirovski and Yudin(1983)]{Nemirovski-1983-Problem}
A.~S. Nemirovski and D.~B. Yudin.
\newblock \emph{Problem Complexity and Method Efficiency in Optimization}.
\newblock Wiley-Interscience, 1983.

\bibitem[Nesterov(2005)]{Nesterov-2005-Smooth}
Y.~Nesterov.
\newblock Smooth minimization of non-smooth functions.
\newblock \emph{Mathematical Programming}, 103\penalty0 (1):\penalty0 127--152,
  2005.

\bibitem[Nesterov(2012)]{Nesterov-2012-Make}
Y.~Nesterov.
\newblock How to make the gradients small.
\newblock \emph{Optima. Mathematical Optimization Society Newsletter},
  88:\penalty0 10--11, 2012.

\bibitem[Nesterov(2018)]{Nesterov-2018-Lectures}
Y.~Nesterov.
\newblock \emph{Lectures on Convex Optimization}, volume 137.
\newblock Springer, 2018.

\bibitem[Outrata et~al.(1998)Outrata, Kocvara, Zowe, and
  Zowe]{Outrata-1998-Nonsmooth}
J.~Outrata, M.~Kocvara, J.~Zowe, and J.~Zowe.
\newblock \emph{Nonsmooth Approach to Optimization Problems with Equilibrium
  Constraints: Theory, Applications and Numerical Results}, volume~28.
\newblock Springer Science \& Business Media, 1998.

\bibitem[Rockafellar and Wets(2009)]{Rockafellar-2009-Variational}
R.~T. Rockafellar and R.~J-B. Wets.
\newblock \emph{Variational Analysis}, volume 317.
\newblock Springer Science \& Business Media, 2009.

\bibitem[Shamir et~al.(2020)Shamir, Lin, and Coviello]{shamir2020smooth}
G.~I. Shamir, D.~Lin, and L.~Coviello.
\newblock Smooth activations and reproducibility in deep networks.
\newblock \emph{arXiv preprint arXiv:2010.09931}, 2020.

\bibitem[Tatro et~al.(2020)Tatro, Chen, Das, Melnyk, Sattigeri, and
  Lai]{tatro2020optimizing}
N.~Tatro, P.~Chen, P.~Das, I.~Melnyk, P.~Sattigeri, and R.~Lai.
\newblock Optimizing mode connectivity via neuron alignment.
\newblock \emph{Advances in Neural Information Processing Systems},
  33:\penalty0 15300--15311, 2020.

\bibitem[Tian and So(2022)]{Tian-2022-No}
L.~Tian and A.~M-C. So.
\newblock No dimension-free deterministic algorithm computes approximate
  stationarities of {L}ipschitzians.
\newblock \emph{ArXiv Preprint: 2210.06907}, 2022.

\bibitem[Tian et~al.(2022)Tian, Zhou, and So]{Tian-2022-Finite}
L.~Tian, K.~Zhou, and A.~M-C. So.
\newblock On the finite-time complexity and practical computation of
  approximate stationarity concepts of {L}ipschitz functions.
\newblock In \emph{ICML}, pages 21360--21379. PMLR, 2022.

\bibitem[Vavasis(1993)]{Vavasis-1993-Black}
S.~A. Vavasis.
\newblock Black-box complexity of local minimization.
\newblock \emph{SIAM Journal on Optimization}, 3\penalty0 (1):\penalty0 60--80,
  1993.

\bibitem[Zhang et~al.(2020)Zhang, Lin, Jegelka, Sra, and
  Jadbabaie]{Zhang-2020-Complexity}
J.~Zhang, H.~Lin, S.~Jegelka, S.~Sra, and A.~Jadbabaie.
\newblock Complexity of finding stationary points of nonconvex nonsmooth
  functions.
\newblock In \emph{ICML}, pages 11173--11182. PMLR, 2020.

\end{thebibliography}

\appendix

\section{Related Work}\label{sec:related_work}
To appreciate the difficulty and the scope of research agenda in nonsmooth nonconvex optimization, we describe the relevant literature. In this context, existing research is mostly devoted to establishing the asymptotic convergence of optimization algorithms, including the gradient sampling (GS) method~\citep{Burke-2002-Approximating, Burke-2002-Two, Burke-2005-Robust, Kiwiel-2007-Convergence, Burke-2020-Gradient}, bundle methods~\citep{Kiwiel-1996-Restricted} and subgradient methods~\citep{Benaim-2005-Stochastic, Davis-2020-Stochastic, Daniilidis-2020-Pathological, Bolte-2021-Conservative}. More specifically,~\citet{Burke-2002-Approximating} provided the systematic investigation of approximating a generalized gradient through a simple yet novel random sampling scheme, motivating the subsequent development of celebrated gradient bundle method~\citep{Burke-2002-Two}. Then,~\citet{Burke-2005-Robust} and~\citet{Kiwiel-2007-Convergence} proposed the modern GS method by incorporating key modifications into the scheme of the aforementioned gradient bundle method and proved that any cluster point of the iterates generated by the GS method is a Clarke stationary point. For an overview of GS methods, we refer to~\citet{Burke-2020-Gradient}.

There has been recent progress in the investigation of different subgradient methods for nonsmooth nonconvex optimization. It was shown by~\citet{Daniilidis-2020-Pathological} that the standard subgradient method fails to find any Clarke stationary point of a Lipschitz function, as witnessed by the existence of pathological examples.~\citet{Benaim-2005-Stochastic} established the asymptotic convergence guarantee of stochastic approximation methods from a differential inclusion point of view under additional conditions and~\citet{Bolte-2021-Conservative} justified \textit{automatic differentiation} as used in deep learning.~\citet{Davis-2020-Stochastic} proved the asymptotic convergence of subgradient methods if the objective function is assumed to be Whitney stratifiable. Turning to nonasymptotic convergence guarantee,~\citet{Zhang-2020-Complexity} proposed a randomized variant of Goldstein's subgradient method and proved a dimension-independent complexity bound of $\widetilde{O}(\delta^{-1}\epsilon^{-3})$ for finding a $(\delta, \epsilon)$-Goldstein stationary point of a Hadamard directionally differentiable function. For the more broad class of Lipschitz functions,~\citet{Davis-2022-Gradient} and~\citet{Tian-2022-Finite} have proposed two other randomized variants of Goldstein's subgradient method and proved the same complexity guarantee. Comparing to their randomized counterparts, deterministic algorithms are relatively scarce in nonsmooth nonconvex optimization.

In convex optimization, we have a deep understanding of the complexity of finding an $\epsilon$-optimal point (i.e., $\x \in \br^d$ such that $f(\x) - \inf_{\x \in \br^d} f(\x) \leq \epsilon$)~\citep{Nemirovski-1983-Problem, Guzman-2015-Lower, Braun-2017-Lower, Nesterov-2018-Lectures}. In smooth nonconvex optimization, various lower bounds have been established for finding an $\epsilon$-stationary point (i.e., $\x \in \br^d$ such that $\|\nabla f(\x)\| \leq \epsilon$)~\citep{Vavasis-1993-Black, Nesterov-2012-Make, Carmon-2020-Lower, Carmon-2021-Lower}. Further extensions to nonconvex stochastic optimization were given in~\citet{Arjevani-2020-Second, Arjevani-2022-Lower} while algorithm-specific lower bounds for finding an $\epsilon$-stationary point were derived in~\citet{Cartis-2010-Complexity, Cartis-2012-Complexity, Cartis-2018-Worst}. However, these proof techniques can not be extended to nonsmooth nonconvex optimization due to different optimality notions. In this vein,~\citet{Zhang-2020-Complexity} and~\citet{Kornowski-2021-Oracle} have demonstrated that neither an $\epsilon$-Clarke stationary point nor a near $\epsilon$-Clarke stationary point can be obtained in a $\textnormal{poly}(d, \epsilon^{-1})$ number of queries when $\epsilon > 0$ is smaller than some constant. Our analysis is inspired by their construction and techniques but focus on establishing lower bounds for finding a $(\delta, \epsilon)$-Goldstein stationary point.   

The smoothing viewpoint starts with~\citet[Theorem~9.7]{Rockafellar-2009-Variational}, which states that any approximate Clarke stationary point of a Lipschitz function is the asymptotic limit of appropriate approximate stationary points of smooth functions. In particular, given a Lipschitz function $f$, we can attempt to construct a smooth function $\tilde{f}$ that is $\delta$-close to $f$ (i.e., $\|f-g\|_\infty \leq \delta$), and apply a smooth optimization algorithm on $\tilde{f}$. Such smoothing approaches have been used in convex optimization~\citep{Nesterov-2005-Smooth, Beck-2012-Smoothing} and found the application in structured nonconvex optimization~\citep{Chen-2012-Smoothing}. For a general Lipschitz function,~\citet{Duchi-2012-Randomized} proposed a randomized smoothing approach that can transform the original problem to a smooth nonconvex optimization where the objective function is given in the expectation form and the smoothness parameter is dimension-dependent. Moreover, there are deterministic smoothing approaches that yield dimension-independent smoothness parameters but they are computationally intractable~\citep{Lasry-1986-Remark, Attouch-1993-Approximation}. Recently,~\citet{Kornowski-2021-Oracle,kornowski2022oracle} have explored the trade-off between computational tractability and smoothing, ruling out the existence of any (possibly randomized) smoothing approach that achieves computational tractability and a dimension-independent smoothness parameter.

\end{document}